\documentclass[10pt]{article} 
\usepackage[preprint]{tmlr}

\usepackage{amsmath,amsfonts,bm}

\def\eqref#1{equation~\ref{#1}}

\def\1{\bm{1}}

\DeclareMathAlphabet{\mathsfit}{\encodingdefault}{\sfdefault}{m}{sl}
\SetMathAlphabet{\mathsfit}{bold}{\encodingdefault}{\sfdefault}{bx}{n}

\DeclareMathOperator*{\argmax}{arg\,max}

\usepackage{hyperref}
\usepackage{url}

\usepackage{natbib}

\usepackage{algorithm}
\usepackage{algorithmic}

\usepackage{booktabs}
\usepackage{bm}
\usepackage{mathtools}
\usepackage{newtxtext}

\usepackage{amsmath}
\usepackage{dsfont}

\newcommand{\cA}{\mathcal{A}}
\newcommand{\cD}{\mathcal{D}}
\newcommand{\cG}{\mathcal{G}}
\newcommand{\cM}{\mathcal{M}}
\newcommand{\cO}{\mathcal{O}}
\newcommand{\cR}{\mathcal{R}}
\newcommand{\cS}{{\mathcal{S}}}
\newcommand{\cT}{{\mathcal{T}}}
\newcommand{\cZ}{\mathcal{Z}}
\newcommand{\EE}{\mathbb{E}}
\newcommand{\PP}{\mathbb{P}}

\newtheorem{theorem}{Theorem}
\newtheorem{proposition}{Proposition}[section]
\newtheorem{lemma}[theorem]{Lemma}
\newtheorem{corollary}[theorem]{Corollary}

\newtheorem{remark}{Remark}

\newenvironment{proof}{\noindent\textit{Proof.}}{\hfill $\square$}

\newcommand{\eqnref}[1]{Eq.~(\ref{#1})}

\title{Improving Offline Goal-Conditioned Reinforcement Learning via Selective Reward Stimulation}

\author{\name Jing ZHANG \email jzhanggy@connect.ust.hk \\
      \addr The Hong Kong University of Science and Technology}

\begin{document}

\maketitle

\begin{abstract}
Goal-conditioned reinforcement learning aims to learn policies that reach specified goals, but remains challenging in offline settings with sparse rewards and long-horizon dependencies. In such settings, goal-completion information can be temporally distant from the early decisions that enable success, while offline value estimation introduces additional error. We study this issue from a reward-propagation perspective and show, in a stylized delayed-goal setting, how goal-directed value separation can become small relative to local estimation error. Motivated by this analysis, we propose Reward Stimulation Implicit Q-Learning (RSIQL), a simple non-hierarchical method that introduces additional reward signals at progress-making intermediate states in offline trajectories. RSIQL uses an auxiliary goal-conditioned value function to identify intermediate states estimated to make progress toward the goal and applies reward stimulation to provide less-delayed training supervision. Unlike hierarchical methods, RSIQL does not learn a separate high-level subgoal policy. Experiments on D4RL goal-reaching benchmarks and OGBench show that RSIQL improves over goal-conditioned IQL on average and achieves performance competitive with hierarchical offline goal-conditioned methods, while retaining a simple flat policy structure.
\end{abstract}

\section{Introduction}

Goal-conditioned reinforcement learning (GCRL) aims to learn policies that reach specified goals through sequential decision-making~\citep{schaul2015universal,kaelbling1993learning,sutton2018reinforcement}. Many practical problems can be naturally formulated in this way, including robotic manipulation, navigation, autonomous driving, and embodied decision-making~\citep{anderson2018vision, krantz2020beyond, chaplot2020neural}. In these problems, the agent must learn not only how to act in the current state, but also how to adapt its behavior to different desired goals. This makes goal-conditioned reinforcement learning a flexible framework for multi-task and long-horizon control.

A major challenge in goal-conditioned reinforcement learning is delayed goal-completion supervision. In many goal-reaching formulations, non-goal transitions provide little task-specific information, while the informative improvement in reward occurs only when the goal is reached. Under the $0/-1$ convention used in our analysis, for example, all non-goal transitions receive the same step cost and goal completion is reflected by the change from $-1$ to $0$. Consequently, learning must infer which earlier actions contributed to eventual success from temporally delayed goal-completion information. This difficulty becomes more severe as the horizon increases: decisions made far from the goal receive weak and indirect feedback, even when they are necessary for eventual success.

Offline goal-conditioned reinforcement learning further amplifies this challenge. In the offline setting, the agent must learn entirely from a fixed dataset without additional environment interaction~\citep{kaelbling1993hierarchical, kaelbling1993learning, pong2018temporal, nair2018visual}. This makes the problem attractive for applications where online exploration is unsafe, expensive, or impractical. However, it also introduces distribution shift between the behavior policy that collected the data and the learned policy~\citep{lange2012batch, levine2020offline, fujimoto2019off}. Therefore, value learning must simultaneously cope with sparse long-horizon supervision~\citep{gupta2019relay, park2024hiql, nachum2018data} and the risk of extrapolation error from out-of-distribution actions.

\begin{figure*}[ht!]
\begin{center}
\centerline{\includegraphics[width=0.99\columnwidth]{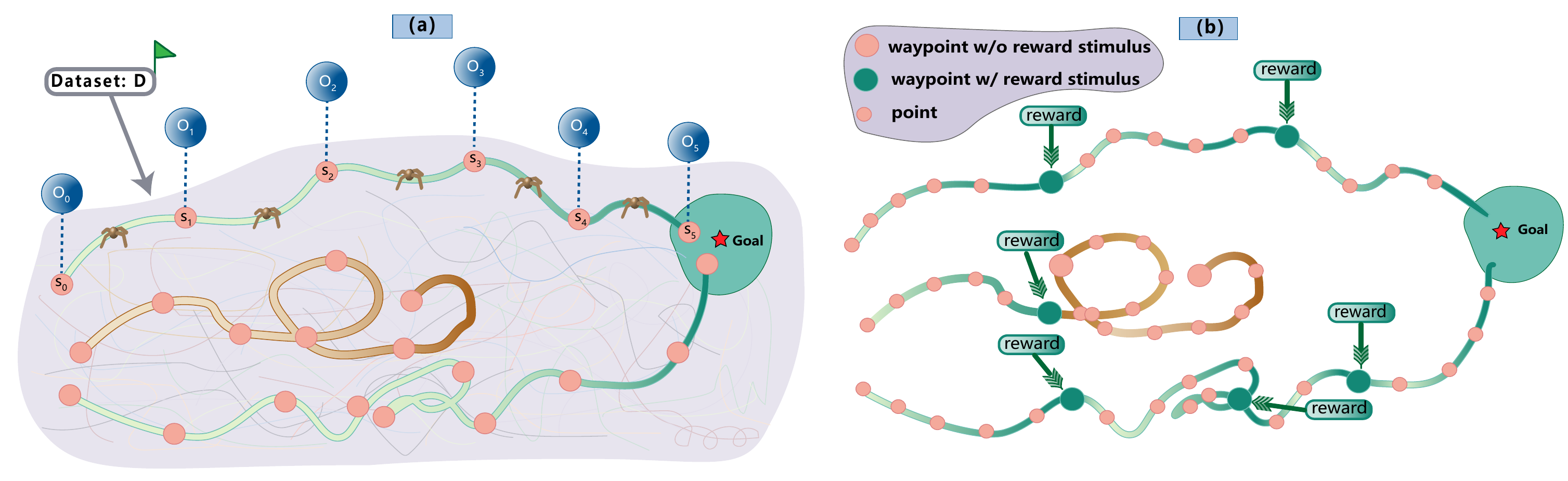}}
\caption{Illustration of delayed goal-completion supervision and RSIQL. (a) Offline trajectories can contain both goal-reaching and suboptimal behavior, while goal-completion information is temporally distant from earlier decisions. (b) RSIQL uses an auxiliary progress criterion to select candidate intermediate states and applies reward stimulation only at transitions associated with estimated progress toward the final goal. Selected transitions therefore provide additional, less-delayed training supervision, whereas candidates that fail the progress criterion receive no stimulation.}
\label{fig-idea}
\end{center}
\end{figure*}

A common strategy for addressing long-horizon goal-conditioned problems is to decompose the final goal into intermediate subgoals. Early goal-conditioned approaches incorporate the goal as a conditioning variable in the value function through the Universal Value Function Approximator (UVFA)~\citep{schaul2015universal}, but this alone does not resolve the difficulty of long-horizon reward propagation. More recent offline GCRL methods often rely on hierarchical reinforcement learning, where a high-level policy predicts intermediate subgoals and a low-level policy aims to reach them~\citep{sutton1999between, nachum2018data, vezhnevets2017feudal, precup2000temporal, park2024hiql}. These methods exploit short-horizon subgoals but still rely on learned values to train or evaluate the high-level policy. More recently, OTA introduces option-aware temporal abstraction directly into value learning, updating the high-level value function over temporally extended options in order to improve long-horizon advantage estimation~\citep{ahn2025ota}. OTA therefore changes the temporal scale of the value update while retaining a hierarchical high-level subgoal policy. Our question is complementary: can the training signal itself be strengthened while retaining both the standard one-step Bellman update and a flat primitive policy? Other approaches use planning or graph-based abstractions to identify intermediate states \citep{choset2005principles, lavalle1998rapidly, kavraki1996probabilistic, srinivas2018universal, qureshi2019motion, paul2019learning, csimcsek2005identifying, mannor2004dynamic}, but constructing and searching such structures can be difficult in continuous high-dimensional spaces~\citep{eysenbach2019search}. 

This motivates a simpler question: \textbf{can intermediate structure improve long-horizon offline goal-conditioned learning without learning an explicit hierarchical policy?} In this paper, we answer this question from the perspective of reward propagation. Under the $0/-1$ goal-reaching convention, we analyze how delayed goal-completion information is propagated backward along offline trajectories and show, in a stylized delayed-goal setting, that useful value separation at early states can become small relative to value-estimation error. This perspective suggests that one way to improve offline goal-conditioned learning is not necessarily to learn a separate subgoal policy, but to strengthen the training signal at informative intermediate states already present in the offline data.

Based on this observation, we propose Reward Stimulation Implicit Q-Learning (RSIQL), a simple non-hierarchical method for offline goal-conditioned reinforcement learning. RSIQL first uses an auxiliary goal-conditioned value function to estimate whether a future state along the same trajectory represents progress toward the final goal. When such a state is identified, RSIQL introduces an additional reward stimulus, thereby providing a less-delayed training signal. The policy and value functions are then trained with an IQL-style offline RL objective using the stimulated rewards.

Unlike hierarchical goal-conditioned methods, RSIQL does not train a high-level policy to predict subgoals at test time. Intermediate states are used only during training to provide additional supervision for value learning. This keeps the learned policy flat and avoids explicit subgoal generation during deployment. At the same time, reward stimulation allows the method to benefit from intermediate trajectory structure in the offline dataset.

We evaluate RSIQL on D4RL and OGBench goal-conditioned benchmarks, including sparse-reward locomotion, manipulation, state-based, and pixel-based tasks. The results show that RSIQL improves over goal-conditioned IQL on average and on many long-horizon sparse-reward tasks, while achieving performance competitive with hierarchical offline goal-conditioned methods. These findings suggest that reward stimulation can provide a simple and effective alternative to explicit subgoal prediction in long-horizon offline goal-conditioned reinforcement learning.

Our main contributions are as follows:

\begin{itemize}
\item We analyze offline goal-conditioned reinforcement learning from the perspective of reward propagation. The analysis identifies horizon-dependent scaling of the goal-reaching message and shows, in a delayed-goal setting, how useful value separation can become small relative to local value-estimation error.

\item We propose Reward Stimulation Implicit Q-Learning (RSIQL), a simple flat offline GCRL algorithm that adds reward stimulation at intermediate states estimated to make progress toward the final goal.

\item RSIQL provides an alternative to explicit hierarchical subgoal prediction. Instead of learning a high-level policy, it uses an auxiliary goal-conditioned value function to identify progress-making states and strengthens the value-learning signal through modified rewards.

\item We evaluate RSIQL on D4RL and OGBench goal-conditioned benchmarks. The results show improvement over goal-conditioned IQL on average and competitive performance relative to hierarchical offline GCRL baselines, especially on several long-horizon sparse-reward tasks.
\end{itemize}

The remainder of the paper is organized as follows. Section \ref{app_relatedwork} reviews related work. Section \ref{sec3} provides background. Section \ref{secPGMs} provides a mechanism-level analysis of delayed goal-completion supervision and local value-estimation error. Section \ref{sec5} presents our method. Section \ref{sec6} reports experimental results on a range of benchmarks. Finally, Sections \ref{sec7} and \ref{sec8} conclude the paper with a summary of contributions, limitations, and broader implications.

\section{Related Work}\label{app_relatedwork}

\subsection{Goal-conditioned reinforcement learning (GCRL)}
Goal-conditioned reinforcement learning studies how to learn policies that achieve specified goals. Compared with standard control tasks, goal-conditioned tasks are often more challenging because rewards are typically sparse and are only observed when the goal is reached \citep{roberts2018special, santucci2016grail, atkeson2015no, kaelbling1993learning}. These challenges become even more pronounced in multi-goal settings, where the policy must generalize across a distribution of goals \citep{kaelbling1993hierarchical, kaelbling1993learning, andrychowicz2017hindsight, pong2018temporal, nair2018visual}.

Prior work addresses these challenges from both the data and algorithmic perspectives. At the data level, a prominent line of work uses goal relabeling to enrich sparse supervision by treating states observed in trajectories as alternative goals~\citep{andrychowicz2017hindsight, fang2019curriculum, bai2019guided, pitis2020maximum, kuang2020goal, li2020generalized, eysenbach2020rewriting, ren2019exploration}. This idea has also been adapted to offline settings, where relabeling is applied directly to static datasets without further environment interaction \citep{yang2023essential, park2024hiql, lynch2020learning, pong2019skew}. Another common approach is reward shaping, which replaces sparse rewards with denser signals based on distances to the goal \citep{trott2019keeping, hartikainen2019dynamical, wu2018laplacian, durugkar2021adversarial, mezghani2023learning} or goal-reaching probabilities \citep{hoang2021successor, machado2020count, ghosh2023reinforcement}. While these methods can substantially improve learning, designing effective shaping signals remains challenging.

At the algorithmic level, a standard formulation is to represent the goal as an additional condition in the value function through the Universal Value Function Approximator (UVFA) \citep{schaul2015universal}. Subsequent work extends this framework by learning richer state-goal representations \citep{ghosh2023reinforcement, park2024hiql, hong2021bi}, latent task spaces \citep{ajay2020opal, lynch2020learning, zeng2024goal}, or model-based planning modules \citep{nair2020goal, charlesworth2020plangan, hafner2019learning, zhang2019solar}. However, model-based approaches remain sensitive to dynamics-modeling errors, which can accumulate over long horizons \citep{finn2017deep, kaiser2019model, babaeizadeh2017stochastic, nagabandi2018neural}.

Beyond value-based GCRL, supervised sequence-modeling approaches have also been explored for goal-reaching tasks \citep{emmons2021rvs, ghosh2019learning, chen2021decision, ding2019goal, yang2022rethinking}. These methods learn to imitate successful trajectory patterns directly, for example through maximum-likelihood objectives \citep{lynch2020learning, ajay2020opal, zeng2024goal}, Transformer-based sequence models \citep{janner2021offline, badrinath2024waypoint}, or latent trajectory models such as VQ-VAE \citep{jiang2022efficient}. Such approaches are often effective when expert or high-quality demonstrations are available.

\subsection{Subgoal discovery and decomposition for long-horizon tasks}
Long-horizon goal-conditioned tasks are particularly difficult because the final reward may be temporally distant from the actions that enable success \citep{nachum2018data, levy2017learning, levy2018hierarchical, park2024hiql, gupta2019relay}. A common strategy is therefore to decompose a distant goal into intermediate subgoals that are easier to reach.

One line of work identifies subgoals through planning or graph-based abstractions. These methods compute paths or landmarks by constructing graphs over the state space~\citep{eysenbach2019search, hoang2021successor, chaplot2020neural, savinov2018semi, zhang2021world}, discretizing the space \citep{paul2019learning, csimcsek2005identifying, mannor2004dynamic}, or performing explicit path planning \citep{choset2005principles, lavalle1998rapidly, kavraki1996probabilistic, srinivas2018universal, qureshi2019motion}. While effective in some domains, these approaches can be difficult to scale to continuous, high-dimensional problems, where graph construction, state partitioning, and distance estimation are all nontrivial. In addition, identifying optimal paths may require stronger structural assumptions or higher-quality data than are typically available in offline datasets.

A second line of work relies on hierarchical reinforcement learning, where a high-level policy predicts intermediate subgoals and a low-level policy attempts to reach them~\citep{sutton1999between, nachum2018data, vezhnevets2017feudal, precup2000temporal}. In offline GCRL, representative examples include methods that infer subgoals using one or more value functions and then optimize lower-level goal-conditioned policies accordingly~\citep{park2024hiql, chane2021goal, levy2017learning, nachum2018data, kulkarni2016hierarchical, gupta2019relay, nair2019hierarchical}. These methods can substantially improve long-horizon performance, but their success depends on accurately identifying useful subgoals, which in turn relies on reliable long-horizon value estimates. In noisy offline settings, subgoal quality may therefore be limited by the same value-estimation errors that make long-horizon GCRL difficult in the first place. Some approaches also parameterize subgoals directly in the state space \citep{park2024hiql, xu2022policy}, which may introduce unnecessary constraints when the goal space is lower-dimensional than the full state space.

A closely related temporal-abstraction approach is Option-aware Temporally Abstracted value learning (OTA)~\citep{ahn2025ota}. OTA diagnoses long-horizon failure in hierarchical offline GCRL through inaccurate high-level advantage estimates and introduces option-aware temporal-difference updates so that the high-level value function is learned over temporally extended action sequences. This reduces the effective horizon of high-level value learning and improves the value ordering used for high-level policy extraction. Unlike HIQL itself, whose primary mechanism is hierarchical decomposition, OTA explicitly introduces temporal abstraction into the value update.

A related direction predicts intermediate subgoals directly from data. For example,~\citep{badrinath2024waypoint} learn waypoints using neural sequence models, while other approaches use heuristics such as selecting intermediate states or minimizing goal-conditioned value objectives \citep{lai2020hindsight, cheng2024goal, nasiriany2019planning}. These methods are naturally compatible with offline data, but they still require an explicit mechanism for identifying or predicting intermediate targets.

In parallel, offline RL has developed a large literature on policy regularization and value conservatism to address distribution shift and out-of-distribution action selection~ \citep{wu2019behavior, fujimoto2021minimalist, kumar2019stabilizing, zhang2024constrained, fujimoto2019off, wang2022diffusion, hansen2023idql, kumar2020conservative, bai2022pessimistic, zhang2024q, kostrikov2021offline, chen2024conservative, garg2023extreme, janner2022planning, chen2021decision}. These techniques are complementary to long-horizon goal decomposition, since even when intermediate subgoals are available, offline methods must still control extrapolation error and prevent out-of-distribution policy updates.

\subsection{Positioning our contribution}
RSIQL is most closely related to reward shaping, offline goal-conditioned value learning, and hierarchical subgoal methods. Its main distinction is that it uses intermediate states as training-time reward-stimulation points, not as test-time subgoals.

Compared with goal-conditioned IQL, RSIQL keeps the same flat goal-conditioned policy structure but modifies the reward signal used for value learning. The goal is to provide less-delayed training supervision in long-horizon sparse-reward tasks.

Compared with reward-shaping methods, RSIQL does not assign dense heuristic rewards to all transitions. Instead, it selectively adds reward stimulation at intermediate states that are estimated to make progress toward the final goal.

Compared with hierarchical methods such as HIQL, RSIQL does not learn a separate high-level policy and does not require subgoal prediction during deployment. The learned policy remains a single goal-conditioned policy. Although hierarchical methods can improve long-horizon offline GCRL by decomposing tasks into subgoals, their effectiveness still depends on reliable subgoal evaluation and reachability estimation. RSIQL avoids this test-time dependency by using intermediate states only during training to strengthen the value-learning signal.

RSIQL is also closely related to OTA~\citep{ahn2025ota}, which addresses long-horizon value-estimation difficulty by incorporating option-aware temporal abstraction into the value update. OTA learns the value used by a hierarchical high-level policy over temporally extended options, thereby changing the temporal scale of temporal-difference learning. RSIQL takes a different route: it keeps the standard primitive one-step Bellman update unchanged and instead modifies the training reward selectively at transitions whose future states are estimated to make progress toward the final goal. Thus, OTA introduces temporal abstraction into value learning and retains hierarchical subgoal prediction, whereas RSIQL introduces less-delayed reward supervision on primitive transitions and deploys a single flat goal-conditioned policy.

Compared with waypoint and sequence-modeling methods, RSIQL does not generate future waypoints or model full trajectories. It remains a TD-based offline RL method and uses intermediate trajectory structure only to strengthen value learning.

\section{Preliminaries}\label{sec3}

\subsection{Offline goal-conditioned problem}\label{offlineGCRL}

\noindent We model the offline goal-conditioned task as an augmented Markov decision process (A-MDP),
\[
\mathcal{M}_A=\{\cS,\cA,\PP_s,r,p_0,\gamma,\cG\},
\]
where $\cS$ is the state space, $\cA$ is the action space, and $\cG$ is the goal space. The initial-state distribution conditioned on goal $g\in\cG$ is denoted by $p_0(s_0\mid g)$, and the transition dynamics are given by the kernel $\PP_s(s'\mid s,a,g)$. The reward function is
\[
r:\cS\times\cA\times\cG\mapsto\{0,1\},
\]
and $\gamma\in(0,1)$ is the discount factor. Many goal-reaching benchmarks report success using a $\{0,1\}$ reward or success metric.
For the theoretical analysis and RSIQL reward construction, we use the $0/-1$ goal-reaching step-cost convention, where successful goal completion receives reward $0$ and non-goal transitions receive reward $-1$.

In the offline setting, learning is performed from a static dataset
\[
\cD=\{\tau_i \mid \tau_i=(s_0,a_0,r_0,g,s_1,a_1,r_1,g,\dots)\},
\]
generated by a behavior policy $\pi_\beta$.\footnote{For simplicity, we write each trajectory using a single goal $g$, but the same analysis applies to multi-goal settings by considering datasets of the form $\cD=\{\tau_i \mid \tau_i=(s_0,a_0,r_0,g_0,s_1,a_1,r_1,g_1,\dots)\}$.} The support of $\cD$ determines the effective state, action, and goal spaces available for offline learning. Depending on the task, the goal space $\cG$ may coincide with the state space or may be different from it.

The objective of offline GCRL is to learn a goal-conditioned policy $\pi(a\mid s,g)$ that maximizes long-term discounted return. For a fixed policy $\pi$, the goal-conditioned state-value function is
\begin{align}\label{eq_uvfa}
V^\pi(s_t,g)
=
\EE_{\pi,\PP_s}\!\left[\sum_{k=0}^{\infty}\gamma^k r_{t+k}\mid s_t,g\right],
\end{align}
and the corresponding goal-conditioned action-value function is
\begin{align}\label{eq_UQFA}
Q^\pi(s_t,a_t,g)
=
\EE_{\pi,\PP_s}\!\left[\sum_{k=0}^{\infty}\gamma^k r_{t+k}\mid s_t,a_t,g\right].
\end{align}
These are the goal-conditioned counterparts of the Universal Value Function Approximator (UVFA) and Universal Q-value Function Approximator (UQFA)~\citep{schaul2015universal}. The optimal action-value function $Q^*$ is instead characterized by the Bellman optimality fixed point
\begin{align}\label{eq_BellmanR}
Q^*(s,a,g)=(\cT_g Q^*)(s,a,g),
\end{align}
where the goal-conditioned Bellman optimality operator is
\begin{align}\label{eq_BellmanO}
\cT_gQ(s,a,g):=r(s,a,g)+\gamma \EE_{s'\sim \PP_s(\cdot\mid s,a,g)}\!\left[\max_{a'}Q(s',a',g)\right].
\end{align}
In some tasks, the goal $g$ may be mapped to a latent representation $f(g)$; in that case, the conditional goal variable in the above definitions is replaced by $f(g)$.

\subsection{Goal-conditioned Implicit Q-Learning (GCIQL)}\label{seciql}

\noindent Implicit Q-Learning (IQL) \citep{kostrikov2021offline} is an offline RL algorithm designed to address Q-value overestimation caused by distribution shift. The key idea of IQL is to use expectile regression when solving the Bellman equation, thereby avoiding the instability introduced by directly maximizing over actions. To this end, IQL learns a separate value function $V$ to replace the maximized Q target in \eqnref{eq_BellmanO}.

Specifically, for the value function $V_\psi(s)$ with parameter $\psi$, IQL minimizes the expectile regression loss
\begin{align}\label{eq_iql_v}
L_V(\psi)=\EE_{(s,a)\sim \cD}\!\left[L_2^\tau\!\left(Q_{\hat{\theta}}(s,a)-V_\psi(s)\right)\right].
\end{align}
For the Q-value function $Q_\theta(s,a)$ with parameter $\theta$, IQL replaces the maximized Bellman target with $V_\psi$ and optimizes
\begin{align}\label{eq_iql_q}
L_Q(\theta)=\EE_{(s,a,s')\sim \cD}\!\left[\bigl(r(s,a)+\gamma V_\psi(s')-Q_\theta(s,a)\bigr)^2\right].
\end{align}
Finally, IQL adopts an AWR-style policy update~\citep{peng2019advantage}. For the policy $\pi_\phi(a\mid s)$ with parameter $\phi$, the policy loss is
\begin{align}\label{eq_iql_p}
L_\pi(\phi)=-\EE_{(s,a)\sim \cD}\!\left[\exp\!\bigl(\beta(Q_{\hat{\theta}}(s,a)-V_\psi(s))\bigr)\log \pi_\phi(a\mid s)\right].
\end{align}

In goal-conditioned tasks, the corresponding goal-conditioned value and Q-value functions replace the standard $V$ and $Q$ functions in IQL. Likewise, the standard policy $\pi_\phi(a\mid s)$ is replaced by the goal-conditioned policy $\pi_\phi(a\mid s,g)$, while the overall optimization procedure remains unchanged.

\section{Mechanistic Analysis of Reward Propagation in Offline Goal-Conditioned Reinforcement Learning}
\label{secPGMs}

In this section, we study offline goal-conditioned reinforcement learning through a mechanism-level analysis of reward propagation. Our goal is not to provide a general failure theorem for offline GCRL. Instead, we isolate two effects that are relevant in long-horizon goal-reaching problems: the absolute goal-reaching message induced by the $0/-1$ reward convention contains a horizon-dependent attenuation factor, and value separation that is itself attenuated with temporal distance can become increasingly vulnerable to local estimation error.

The analysis proceeds in three steps. First, we characterize the horizon-dependent scaling of the goal-reaching backward message under the standard control-as-inference transformation. Second, we connect the corresponding optimality-conditioned action posterior to behavior-policy reweighting, clarifying why relative value estimates are relevant for advantage-weighted policy extraction. Third, in a delayed-goal setting where success and failure receive the same non-goal step cost until the terminal decision time, we derive the resulting temporal attenuation of goal-reaching value separation and study how local approximation and Bellman errors can obscure that separation. These results provide a mechanistic motivation for the reward-stimulation method introduced in Section~\ref{sec5}.

\subsection{Horizon-dependent scaling of the goal-reaching message}
\label{sec4.1}

We begin with the probabilistic-inference view of reinforcement learning \citep{levine2018reinforcement}. We introduce a binary optimality variable $\cO_t$, where $\cO_t=1$ indicates that the transition at time $t$ is compatible with goal-reaching optimal behavior. For a goal $g$, define
\begin{align}
p(\cO_t=1 \mid s_t,a_t,g)
=
\exp(r(s_t,a_t,g)).
\label{eq4.11_revised}
\end{align}
For the analysis, we use the $0/-1$ goal-reaching reward convention
\[
r(s_t,a_t,g)=
\begin{cases}
0, & \text{if the goal is reached},\\
-1, & \text{otherwise}.
\end{cases}
\]
Thus, non-goal transitions incur a unit step penalty, while reaching the goal removes this penalty. Under the control-as-inference transformation,
\begin{align}
p(\cO_t=1 \mid s_t,a_t,g)
=
\begin{cases}
1, & \text{if the goal is reached},\\
e^{-1}, & \text{otherwise}.
\end{cases}
\end{align}

The optimality-conditioned trajectory distribution can then be written as
\begin{align}
p\!\left(\tau \mid \prod_{i=1}^{T}\cO_i=1, g\right)
&\propto
p\!\left(\tau,\prod_{i=1}^{T}\cO_i=1 \mid g\right)
\nonumber\\
&=
p(s_1\mid g)
\prod_{i=1}^{T}
p(\cO_i=1\mid s_i,a_i,g)
p(s_{i+1}\mid s_i,a_i,g)
\pi_{\beta}(a_i\mid s_i,g)
\nonumber\\
&=
p(\tau\mid g)
\exp\!\left(\sum_{i=1}^{T} r(s_i,a_i,g)\right).
\label{eq4.12_revised}
\end{align}
Here, $\pi_{\beta}$ denotes the behavior policy that generated the offline dataset. The expression above shows that optimality conditioning reweights behavior trajectories according to their cumulative goal-conditioned reward.

We define the backward message from a state-action pair $(s_t,a_t)$ to the future optimality event as
\begin{align}
\cM(s_t,a_t,g)
=
p\!\left(\prod_{i=t}^{T}\cO_i=1 \mid s_t,a_t,g\right),
\label{eq4.13_revised}
\end{align}
and the corresponding state-level message as
\begin{align}
\cM(s_t,g)
&=
\int_{\cA}
\cM(s_t,a_t,g)\pi_{\beta}(a_t\mid s_t,g)\,da_t.
\label{eq4.14_revised}
\end{align}
To isolate the component associated with eventual goal achievement, define
\[
\cM^g(s_t,a_t,g)
=
p\!\left(s_T=g,\prod_{i=t}^{T}\cO_i=1 \mid s_t,a_t,g\right),
\]
and
\[
\cM^g(s_t,g)
=
\int_{\cA}\cM^g(s_t,a_t,g)\pi_\beta(a_t\mid s_t,g)\,da_t.
\]

The following proposition characterizes the horizon-dependent factor in this goal-reaching message. The proof is provided in Appendix~\ref{app_a}.

\begin{proposition}[Horizon-dependent scaling of the goal-reaching backward message]
\label{prop1_revised}
Consider an offline goal-conditioned trajectory stopped at terminal time $T$ under the $0/-1$ goal-reaching reward convention above, with successful trajectories stopped when the goal is first reached. Let $\PP_{\pi_{\beta}}(s_T=g\mid s_t,a_t,g)$ denote the probability, under the behavior policy and environment dynamics, that the stopped trajectory reaches goal $g$ at time $T$ from $(s_t,a_t)$. Then
\begin{align}
\cM^g(s_t,a_t,g)
=
\PP_{\pi_{\beta}}(s_T=g\mid s_t,a_t,g)\cdot e^{-(T-t)}.
\label{eq_a2_revised}
\end{align}
Similarly,
\begin{align}
\cM^g(s_t,g)
=
\PP_{\pi_{\beta}}(s_T=g\mid s_t,g)\cdot e^{-(T-t)}.
\label{eq_a3_revised}
\end{align}
At the terminal step, $\cM^g(s_T,a_T,g)=1$ when $s_T$ reaches $g$.
\end{proposition}

Proposition~\ref{prop1_revised} identifies an exponential factor in the \emph{absolute} goal-reaching message induced by the control-as-inference transformation and this reward scale. Importantly, $e^{-(T-t)}$ is shared by actions evaluated at the same time step. The proposition therefore does not, by itself, imply that relative action preferences or advantage differences must vanish with the horizon. Relative action discrimination also depends on the action-dependent reachability term and on how accurately the relevant value quantities are learned from offline data. We make this distinction explicit in the next two subsections.

\subsection{Optimality-conditioned policy extraction}
\label{sec4.3}

We next examine how the backward-message representation induces an action posterior. Since learning is performed from a fixed dataset, this posterior is defined only where the behavior distribution has support. For this subsection, define the analytic log-message quantities
\begin{align}
Q^{\mathrm{msg}}(s_t,a_t,g)=\log \cM(s_t,a_t,g),
\qquad
V^{\mathrm{msg}}(s_t,g)=\log \cM(s_t,g).
\label{eq4.32_revised}
\end{align}
These quantities are induced by the probabilistic-inference construction; they are not the learned IQL critic and value networks. For any $(s_t,g)$ with \(\cM(s_t,g)>0\), the posterior action distribution conditioned on future optimality is
\begin{align}
p\!\left(a_t\mid s_t,g,\prod_{i=t}^{T}\cO_i=1\right)
&=
\frac{
p\!\left(\prod_{i=t}^{T}\cO_i=1\mid s_t,a_t,g\right)
\pi_{\beta}(a_t\mid s_t,g)
}{
p\!\left(\prod_{i=t}^{T}\cO_i=1\mid s_t,g\right)
}
\nonumber\\
&=
\exp\!\left(Q^{\mathrm{msg}}(s_t,a_t,g)-V^{\mathrm{msg}}(s_t,g)\right)
\pi_{\beta}(a_t\mid s_t,g).
\label{eq4.33_revised}
\end{align}
The proof of the following proposition is given in Appendix~\ref{app_optpolicy}.

\begin{proposition}[Optimality-conditioned behavior reweighting]
\label{prop2_revised}
For any state--goal pair $(s,g)$ with \(\cM(s,g)>0\), the optimality-conditioned action posterior over actions in the support of $\pi_\beta(\cdot\mid s,g)$ is
\begin{align}
\pi^{\mathrm{post}}(a\mid s,g)
=
\exp\!\left(Q^{\mathrm{msg}}(s,a,g)-V^{\mathrm{msg}}(s,g)\right)
\pi_{\beta}(a\mid s,g).
\label{eq_policy_reweight}
\end{align}
Equivalently, a parametric policy can be fitted to this posterior through
\begin{align}
\pi^{\mathrm{post}}
\in
\argmax_{\pi}
\;
\EE_{(s,a,g)\sim\cD}
\left[
\exp\!\left(Q^{\mathrm{msg}}(s,a,g)-V^{\mathrm{msg}}(s,g)\right)
\log \pi(a\mid s,g)
\right].
\label{optpolicy_revised}
\end{align}
\end{proposition}

The weighting in \eqnref{optpolicy_revised} has the same functional form as advantage-weighted regression \citep{peng2019advantage}. This provides a structural connection to IQL-style policy extraction, which reweights behavior actions using learned relative value estimates while avoiding explicit maximization over out-of-distribution actions. Proposition~\ref{prop2_revised} is not an equivalence between $Q^{\mathrm{msg}},V^{\mathrm{msg}}$ and the learned IQL value functions, nor does $\pi^{\mathrm{post}}$ denote an unrestricted environment-optimal policy. Its role is to clarify why reliable relative value estimates are important for behavior-supported policy reweighting.

\subsection{Delayed-goal value separation and local signal quality}
\label{sec4.2}

Proposition~\ref{prop1_revised} concerns the absolute backward message and does not, by itself, imply that relative action or policy separation deteriorates with horizon. We therefore analyze value separation directly in the delayed-goal setting underlying the reward-propagation argument. In this setting, the compared behaviors receive the same non-goal step reward until the terminal decision time $T$, and differ only in their probability of reaching the goal at $T$.

Let $Q_t^\pi(s,a,g)$ denote the goal-conditioned action-value function at time step $t$ under policy $\pi$. For two continuation policies $\pi^s$ and $\pi^f$, define the effective progress separation
\begin{align}
\cR_t(s,a,g)
=
\left|
Q_t^{\pi^s}(s,a,g)
-
Q_t^{\pi^f}(s,a,g)
\right|.
\label{eq_b2_new_revised}
\end{align}
The following lemma makes explicit when this separation inherits a geometric dependence on the remaining horizon. Its proof is provided in Appendix~\ref{app_b}.

\begin{lemma}[Delayed-goal value separation]
\label{lem:delayed_goal_separation}
Consider the $0/-1$ goal-reaching reward convention and a finite-horizon comparison from time $t$ to terminal time $T$. Suppose that, under both continuation policies $\pi^s$ and $\pi^f$, the trajectory remains outside the goal for times $j=t,\ldots,T-1$, so both receive reward $-1$ over these steps. Let
\begin{align}
p_s &= \PP_{\pi^s}(s_T=g\mid s_t,a_t,g),\\
p_f &= \PP_{\pi^f}(s_T=g\mid s_t,a_t,g)
\end{align}
be their terminal goal-reaching probabilities. Then
\begin{align}
Q_t^{\pi^s}(s_t,a_t,g)-Q_t^{\pi^f}(s_t,a_t,g)
=
\gamma^{T-t}(p_s-p_f),
\label{eq:delayed_goal_gap_exact}
\end{align}
and therefore
\begin{align}
\cR_t(s_t,a_t,g)
=
\gamma^{T-t}|p_s-p_f|
\le
\gamma^{T-t}.
\label{eq_signal_decay_revised}
\end{align}
In particular, if $\pi^s$ reaches the goal at $T$ with probability one and $\pi^f$ fails with probability one, then $\cR_t(s_t,a_t,g)=\gamma^{T-t}$.
\end{lemma}

Lemma~\ref{lem:delayed_goal_separation} is distinct from Proposition~\ref{prop1_revised}. Proposition~\ref{prop1_revised} identifies a common horizon-dependent factor in an absolute control-as-inference message, whereas Lemma~\ref{lem:delayed_goal_separation} directly characterizes the value separation between delayed goal-reaching and failing behavior under the stated reward-timing condition. The lemma does not claim that arbitrary successful and failing policies satisfy this bound; if one behavior can reach the goal substantially earlier than $T$, the return gap need not be bounded by $\gamma^{T-t}$.

We next study how this delayed-goal separation interacts with local value-estimation error. Let $Q_t$ denote the target action-value quantity being approximated locally and let $\hat Q_t$ denote the learned critic. We decompose the local error into
\begin{align}
\epsilon_t(s,a,g)
&=
Q_t(s,a,g)-\hat Q_t(s,a,g),\\
\delta_t(s,a,g)
&=
\hat Q_t(s,a,g)-\cT_g\hat Q_t(s,a,g),
\end{align}
where $\cT_g$ is the goal-conditioned Bellman operator. Under the local residual model used below, define the aggregate error magnitude
\begin{align}
\zeta_t(s,a,g)
=
\left|
\epsilon_t(s,a,g)+|\delta_t(s,a,g)|
\right|,
\end{align}
and the signal-to-noise ratio
\begin{align}
\Delta_t(s,a,g)
=
\frac{\cR_t(s,a,g)}{\zeta_t(s,a,g)}.
\label{eq_b1_revised}
\end{align}
A value $\Delta_t>1$ means that the progress separation exceeds this local error magnitude.

The following proposition combines Lemma~\ref{lem:delayed_goal_separation} with the stated local error model. The proof is provided in Appendix~\ref{app_b}.

\begin{proposition}[Local signal-to-noise bound under delayed goal completion]
\label{prop3_revised}
Under the conditions of Lemma~\ref{lem:delayed_goal_separation}, assume that $\epsilon_t(s,a,g)\sim \cG(0,b)$ is Gumbel-distributed with scale $b>0$, while $\delta_t(s,a,g)$ is treated as deterministic in the local analysis. Then
\begin{align}
\PP\!\left(\Delta_t(s,a,g)>1\right)
\le
\frac{2}{e b}\gamma^{T-t}.
\label{eq_snr_bound_revised}
\end{align}
In particular, when $b\ge 2$,
\begin{align}
\PP\!\left(\Delta_t(s,a,g)>1\right)
\le
\frac{\gamma^{T-t}}{e}.
\end{align}
\end{proposition}

Proposition~\ref{prop3_revised} remains a local, stylized signal-quality result because of the Gumbel error model, but its horizon dependence no longer enters as an assumed value-separation condition. Instead, it follows from the delayed-goal reward structure in Lemma~\ref{lem:delayed_goal_separation}. Thus, within this setting, temporal distance reduces the return separation between goal-reaching and failing behavior, while local estimation error can make the reduced separation harder to resolve.

Taken together, the analysis identifies two complementary effects. First, under the control-as-inference representation, the absolute goal-reaching message carries a horizon-dependent attenuation factor. Second, in the delayed-goal setting where compared behaviors receive identical non-goal rewards until time $T$, the goal-reaching value separation is itself proportional to $\gamma^{T-t}$, and local approximation and Bellman errors can obscure this separation. These results are mechanism-level statements under explicit reward-timing and local-error conditions rather than a general failure theorem for offline GCRL. They motivate introducing less-delayed training supervision at informative intermediate states, which we instantiate next through reward stimulation.

\section{Reward Stimulation via Intermediate Subgoals for Offline Goal-Conditioned Reinforcement Learning}
\label{sec5}

The analysis in Section~\ref{secPGMs} provides a mechanism-level motivation for strengthening intermediate supervision. Under the control-as-inference representation, the absolute goal-reaching message contains a horizon-dependent attenuation factor. In the delayed-goal setting, the value separation between goal-reaching and failing behavior is also proportional to $\gamma^{T-t}$, so local estimation error can increasingly obscure this distinction at earlier states. This motivates adding less-delayed training supervision at informative intermediate states without requiring a separate high-level policy.

Motivated by this observation, we propose \emph{Reward Stimulation Implicit Q-Learning} (RSIQL), a simple non-hierarchical method for offline GCRL. RSIQL uses intermediate states already present in offline trajectories to provide additional reward supervision during training. The key idea is to identify future states that represent measurable progress toward the final goal and to stimulate the reward of transitions that move toward such states. This introduces less-delayed training supervision while preserving a flat goal-conditioned policy.

RSIQL has two components. First, it pretrains an auxiliary goal-conditioned value function to estimate whether a future state along the same trajectory is closer to the final goal. Second, it uses this auxiliary value function to construct a stimulated reward for IQL-style value learning and policy extraction. Importantly, RSIQL does not train a high-level policy and does not predict subgoals during evaluation. Intermediate states are used only during training to improve the value-learning signal.

\subsection{Selecting Effective Intermediate Subgoals}
\label{sec5.1}

We first describe how RSIQL identifies intermediate states that are useful for reward stimulation. Consider an offline trajectory
\[
\tau=(s_0,a_0,r_0,s_1,\ldots,s_T)
\]
and a target goal $g$. For a state $s_t$, we consider the $k$-step future state $s_{t+k}$ as a candidate intermediate state, where $k$ is a fixed subgoal interval and $t+k\le T$. Intuitively, $s_{t+k}$ is useful if it represents progress from $s_t$ toward the final goal $g$.

To estimate progress, RSIQL uses an auxiliary goal-conditioned value function $\bar V_{\omega}(s,g)$. This auxiliary value function is pretrained before reward stimulation and is not updated using stimulated rewards. This separation is important: if the same value function were used both to select subgoals and to learn from stimulated rewards, the subgoal-selection criterion could become circular. In our implementation, $\bar V_{\omega}$ is trained using the GC-IVL objective \citep{park2024hiql}, which provides a goal-conditioned temporal-distance signal without requiring policy improvement.

The simplest progress criterion is
\begin{align}
\bar V_{\omega}(s_{t+k},g) > \bar V_{\omega}(s_t,g).
\label{eq:progress_basic}
\end{align}
Since larger values correspond to states estimated to be closer to the goal under the $0/-1$ sparse-reward convention, \eqnref{eq:progress_basic} selects future states that are predicted to make progress toward $g$.

In practice, offline datasets often contain suboptimal or detour trajectories. To avoid stimulating weak or noisy progress signals, we use a conservative version of the criterion:
\begin{align}
\eta_t(g;k,\delta)=
\mathds{1}
\left[
\bar V_{\omega}(s_{t+k},g)-\bar V_{\omega}(s_t,g)
>
\delta C_k
\right],
\label{eq:progress_indicator}
\end{align}
where $\eta_t(g;k,\delta)$ is the progress indicator, $\delta \geq 0$ controls the strictness of subgoal selection, and
$C_k = \frac{1-\gamma^k}{1-\gamma}=\sum_{i=0}^{k-1}\gamma^i$ is a discount-dependent $k$-step scale factor. Multiplying this factor by $\delta$ makes the progress threshold increase with the temporal interval $k$, rather than applying the same raw value-difference threshold to candidate states at different temporal separations.
In all main experiments, we use the same fixed hyperparameters $k=25$ and $\delta=0.6$ across tasks.
The additional values considered in Section~6.4 are used only for post-hoc sensitivity analysis and are not used to select the reported main results.

This selection rule is local to the sampled offline trajectory. RSIQL therefore avoids graph construction, global planning, or explicit high-level subgoal prediction. The selected intermediate states are used only to determine where reward stimulation should be applied during training.

\subsection{Reward Stimulation at Effective Intermediate Subgoals}
\label{sec5.2}

After identifying progress-making intermediate states, RSIQL constructs a stimulated reward for value learning. We use the $0/-1$ goal-reaching step-cost convention from Section~\ref{secPGMs}, where $r_t=0$ when the goal is reached and $r_t=-1$ otherwise. Under this convention, assigning reward $0$ to a non-goal progress transition is equivalent to adding a positive reward stimulus of magnitude $1$ relative to the original reward $-1$.

Let $\eta_t(g;k,\delta)$ be the progress indicator in \eqnref{eq:progress_indicator}. We define the stimulated reward as
\begin{align}
\tilde r_t
=
r_t + \eta_t(g;k,\delta)(r_{\max}-r_t),
\label{eq:stimulated_reward_general}
\end{align}
where $r_{\max}=0$ under the $0/-1$ convention, which gives
\begin{align}
\tilde r_t
=
\begin{cases}
0, & r_t=0,\\
0, & r_t=-1 \text{ and } \eta_t(g;k,\delta)=1,\\
-1, & r_t=-1 \text{ and } \eta_t(g;k,\delta)=0.
\end{cases}
\label{eq5.21_revised}
\end{align}

\eqnref{eq5.21_revised} makes explicit that RSIQL does not assign reward to all intermediate states. Reward stimulation is applied only when the $k$-step future state is estimated to make sufficient progress toward the final goal. This selectivity is important because stimulating ineffective intermediate states may bias value learning toward detours or subgoals that do not help reach the final goal.

Importantly, setting $\tilde r_t=0$ at a selected non-terminal transition does not treat that transition as terminal or as equivalent to reaching the goal. RSIQL retains the standard one-step Bellman bootstrap, so the critic target at a selected non-terminal transition is
\begin{align}
 y_t^{\mathrm{stim}}=0+\gamma V_\psi(s_{t+1},g),
 \label{eq:stimulated_bellman_target}
\end{align}
whereas an unstimulated non-terminal transition uses
\begin{align}
 y_t^{\mathrm{orig}}=-1+\gamma V_\psi(s_{t+1},g).
 \label{eq:original_bellman_target}
\end{align}
Thus, reward stimulation removes the current non-goal step penalty while preserving the discounted continuation value of the successor state. It neither truncates the trajectory nor replaces the future return by a terminal reward.

The stimulated reward changes the learning signal but not the evaluation objective. During evaluation, the learned policy is still judged by the original goal-reaching success criterion. Thus, reward stimulation should be understood as a training-time credit-assignment mechanism rather than a change to the task definition.

\begin{proposition}[Effect of reward stimulation on training returns]
\label{thm:rsiql_effect}
Let $Q_t^\pi(s_t,a_t,g)$ and $\tilde Q_t^\pi(s_t,a_t,g)$ denote the action-value functions under the original reward $r$ and the stimulated reward $\tilde r$, respectively, for the same policy $\pi$. Suppose $\tilde r_t$ is defined by \eqnref{eq:stimulated_reward_general}. Then, for any state-action pair $(s_t,a_t)$,
\begin{align}
\tilde Q_t^\pi(s_t,a_t,g)-Q_t^\pi(s_t,a_t,g)
=
\mathbb{E}_{\pi}
\left[
\sum_{j=t}^{T}
\gamma^{j-t}
 \eta_j(g;k,\delta)(r_{\max}-r_j)
\mid s_t,a_t,g
\right].
\label{eq:stimulated_return_gap}
\end{align}
Consequently,
\begin{align}
\tilde Q_t^\pi(s_t,a_t,g)\ge Q_t^\pi(s_t,a_t,g).
\end{align}
Moreover, the inequality is strict whenever a progress-making transition with $\eta_j(g;k,\delta)=1$ and $r_j<r_{\max}$ is encountered with nonzero probability.
\end{proposition}

\begin{corollary}[Less-delayed training signal from reward stimulation]
\label{cor:shorter_horizon}
Suppose a delayed-goal trajectory remains non-goal through time $T-1$ and first reaches the goal at time $T$. If it contains a selected non-goal progress transition at time $t_i$, where $t \le t_i<T$, $\eta_{t_i}(g;k,\delta)=1$, and $r_{t_i}<r_{\max}$, then the stimulated return from $(s_t,a_t)$ contains an additional contribution
\[
\gamma^{t_i-t}(r_{\max}-r_{t_i}).
\]
Under the original $0/-1$ goal-reaching convention, the goal-completion improvement over the non-goal step penalty is first observed at time $T$ and is discounted by $\gamma^{T-t}$. Since $t_i<T$ and $\gamma\in(0,1)$,
\[
\gamma^{t_i-t} > \gamma^{T-t}.
\]
Thus, reward stimulation introduces a less-delayed positive training contribution. This statement concerns the temporal location of reward information; RSIQL still uses the standard one-step Bellman update.
\end{corollary}

The proof details of Proposition~\ref{thm:rsiql_effect} and Corollary~\ref{cor:shorter_horizon} are postponed to Appendix~\ref{app_reward_stimulation}.

Proposition~\ref{thm:rsiql_effect} does not claim that RSIQL increases the optimal value under the original task reward. Rather, it shows that reward stimulation locally removes selected non-goal step penalties while the standard discounted continuation-value bootstrap is preserved. Since IQL extracts policies using advantage-weighted regression, the resulting less-delayed value targets on progress-making transitions can produce more informative advantage weights for goal-directed policy learning. Whether this modified training signal improves performance under the original evaluation reward is an empirical question studied in Section~\ref{sec6}.

Proposition~\ref{thm:rsiql_effect} and Corollary~\ref{cor:shorter_horizon} are
training-signal results rather than policy-invariance guarantees. RSIQL modifies
the reward used for training and therefore should not be interpreted as
potential-based reward shaping, which is designed to preserve optimal policies
under specific reward transformations \citep{ng1999policy}. Instead, RSIQL uses
selective reward stimulation to improve credit assignment in offline sparse-reward
goal-conditioned learning. All policies are evaluated using the original
goal-reaching metric in Section~\ref{sec6}. 

Algorithm~\ref{alg:RSIQL} summarizes the complete training procedure of RSIQL. The method first pretrains an auxiliary progress value function $\bar{V}$ to evaluate whether intermediate states are effective subgoals for reaching the final goal. It then applies reward stimulation to those subgoals that are estimated to promote progress toward the goal, and finally optimizes the policy and value functions using the standard IQL update rules under the stimulated rewards.

\begin{algorithm}[ht]
   \caption{Reward Stimulation Implicit Q-Learning (RSIQL)}
   \label{alg:RSIQL}
\begin{algorithmic}
   \STATE {\bfseries Input:} offline dataset $\mathcal D$, subgoal interval $k$, strictness parameter $\delta$, target-network update rate $\lambda$.
   \STATE {\bfseries Initialize:} Q-functions $Q_{\theta_1},Q_{\theta_2}$, value function $V_{\psi}$, policy $\pi_{\phi}$, auxiliary value function $\bar V_{\omega}$, and target Q-functions $Q_{\theta_1'},Q_{\theta_2'}$.

   \STATE \textbf{Pretrain auxiliary progress value function}
   \FOR{each pretraining step}
      \STATE Sample goal-conditioned data $(s,g)$ from $\mathcal D$.
      \STATE Update $\bar V_{\omega}(s,g)$ using the GC-IVL~\citep{park2024hiql} objective.
   \ENDFOR

   \STATE \textbf{Train RSIQL with reward stimulation}
   \FOR{each training step}
      \STATE Sample a minibatch of trajectory segments $(s_t,a_t,r_t,s_{t+1},s_{t+k},g)$ from $\mathcal D$.
      \STATE Compute progress indicator $\eta_t(g;k,\delta)$ by \eqnref{eq:progress_indicator}
      \STATE Construct stimulated reward $\tilde r_t$ by \eqnref{eq:stimulated_reward_general}
      \STATE {\bfseries Update:}
      \STATE Update $V_{\psi}(s,g)$ using the GCIQL value loss \eqnref{eq_iql_v}.
      \STATE Update $\pi_{\phi}(a\mid s,g)$ using the advantage-weighted GCIQL policy loss.
      \STATE Update $Q_{\theta_1},Q_{\theta_2}$ using the stimulated Bellman target
$\tilde r_t+\gamma V_\psi(s_{t+1},g)$.
      
      \STATE Update target networks: $\theta_i' \leftarrow \lambda \theta_i+(1-\lambda)\theta_i', \; i=1,2$.

   \ENDFOR
\end{algorithmic}
\end{algorithm}

\section{Experiments}
\label{sec6}

We evaluate RSIQL on offline goal-conditioned reinforcement learning benchmarks with sparse rewards and long-horizon dependencies. The experiments are designed to answer four questions:

\begin{enumerate}
    \item Does reward stimulation improve over a flat goal-conditioned IQL baseline?
    \item Can RSIQL match or improve upon hierarchical offline GCRL methods without learning a high-level subgoal policy?
    \item How sensitive is RSIQL to the subgoal interval $k$ and the progress-selection threshold $\delta$?
    \item Which component of RSIQL is responsible for the performance gain?
\end{enumerate}

We consider both state-based and pixel-based environments. For state-based single-goal tasks, we use D4RL~\citep{fu2020d4rl} AntMaze and FrankaKitchen~\citep{gupta2019relay}. For multi-goal and visual tasks, we use OGBench~\citep{park2024ogbench}. All policies are evaluated using the original task reward and success metric; reward stimulation is used only during training.

\subsection{Experimental Setup}
\label{sec6_setup}

\paragraph{Benchmarks.}
We evaluate on D4RL goal-conditioned tasks, including AntMaze and FrankaKitchen. Following prior offline GCRL work, we use the \textit{-v2} datasets for AntMaze-Medium and AntMaze-Large, and the \textit{-v0} AntMaze-Ultra datasets introduced by \citet{jiang2022efficient}, with both play and diverse variants. For Kitchen, we evaluate on the \textit{-v0} partial and mixed datasets. Since the original Kitchen reward counts completed subtasks, our RSIQL runs use a sparser terminal-style training reward that retains only the maximum task-completion reward, making the training objective closer to sparse goal reaching. The matched Kitchen ablations in Table~\ref{tab:ablation_stimulation} use the same sparse reward conversion for every variant. In contrast, the published Kitchen baseline numbers in Table~\ref{tab:d4rl_main} retain the reward protocol used in their source paper, so those cross-paper Kitchen comparisons should be interpreted as contextual rather than strictly reward-matched.

We also evaluate on OGBench, which contains more diverse offline goal-conditioned tasks, including state-based and pixel-based multi-goal environments. For pixel-based tasks, we follow OGBench and use a shared state-goal visual encoder with an IMPALA encoder \citep{espeholt2018impala} for value and policy learning.

\paragraph{Baselines.}
We compare RSIQL with three categories of baselines. First, we include flat goal-conditioned methods, including goal-conditioned behavioral cloning (GCBC)~\citep{ding2019goal,ghosh2019learning}, goal-conditioned implicit value learning (GCIVL)~\citep{park2024hiql}, and goal-conditioned IQL~\citep{kostrikov2021offline} (GCIQL). Second, we include hierarchical or subgoal-based methods, including HGCBC~\citep{gupta2019relay,lynch2020learning}, GC-POR~\citep{xu2022policy}, and HIQL~\citep{park2024hiql}. Third, we include sequence-modeling and planning-style baselines such as Trajectory Transformer~\citep{janner2021offline} and TAP~\citep{jiang2022efficient}. For Table~\ref{tab:d4rl_main}, the GCBC, HGCBC, GCIQL, GC-POR, and HIQL results are taken from the HIQL evaluation~\citep{park2024hiql}; TAP and Trajectory Transformer results follow the values reported there from prior work. For Table~\ref{tab:ogbench_main}, baseline numbers follow the official OGBench reference results~\citep{park2024ogbench}. On OGBench, we additionally compare with contrastive RL (CRL)~\citep{eysenbach2022contrastive} and quasimetric RL (QRL)~\citep{wang2023optimal}.

\paragraph{Evaluation protocol.}
Main RSIQL results and matched reward-stimulation ablations are averaged over six random seeds, while the post-hoc $k$ and $\delta$ sensitivity sweeps use five random seeds as indicated in their figure captions. For OGBench, each random-seed run is evaluated over the benchmark-specified set of goals, with the within-run score following the benchmark aggregation over those goals. For our RSIQL results, we report the mean and standard deviation across the six random-seed runs for each task. Externally reported baselines retain the seed and evaluation protocol of their source results. All learned RSIQL policies are evaluated using the original environment success criterion rather than the stimulated training reward.

\paragraph{Fairness and implementation.}
RSIQL uses the same IQL backbone as GCIQL. The main algorithmic difference is the stimulated reward defined in \eqnref{eq:stimulated_reward_general}. The auxiliary value function used for progress estimation is pretrained before RSIQL training and is not updated using stimulated rewards. All main RSIQL experiments use the standard one-step Bellman update, including the pixel-based OGBench tasks, and use the same fixed core hyperparameters $\tau=0.7$, $\beta=3$, $k=25$, and $\delta=0.6$. Hyperparameters, network architectures, and training details are provided in Appendix~\ref{app_experiment_details}.

\subsection{Main Results on D4RL Goal-Conditioned Benchmarks}
\label{sec6_d4rl}

Table~\ref{tab:d4rl_main} reports performance on D4RL AntMaze and Kitchen tasks. On AntMaze, RSIQL improves over the reported GCIQL result on all six datasets. On Kitchen, RSIQL is trained with the sparse reward conversion described above, whereas the published baselines retain their original source protocol; these Kitchen rows therefore provide contextual performance comparisons rather than a strictly reward-matched causal comparison. The matched reward-stimulation ablation in Table~\ref{tab:ablation_stimulation} provides the controlled comparison against GCIQL.

\begin{table*}[h!]
\caption{Performance of RSIQL on D4RL goal-conditioned tasks. AntMaze-Medium and AntMaze-Large use the \textit{-v2} datasets, AntMaze-Ultra uses the \textit{-v0} datasets, and Kitchen uses \textit{-v0}. GCBC, HGCBC, GCIQL, GC-POR, and HIQL results are from the HIQL evaluation~\citep{park2024hiql}; TAP and TT follow the values reported there from prior work. RSIQL results are our six-seed runs. The published Kitchen baselines use their source reward protocol, whereas RSIQL uses the sparse reward conversion described in Section~\ref{sec6_setup}. Bold marks the highest reported mean in each row; if all methods are tied, no entry is bolded.}
\label{tab:d4rl_main}
\begin{center}
\begin{small}
\begin{tabular}{llcccccccc}
\toprule
&Dataset& GCBC & HGCBC & GCIQL  &  GC-POR & TAP & TT & HIQL(repr.) &RSIQL(Ours)\\
\midrule
&antmaze-medium-diverse    &67.3{\tiny $\pm$10.1} &71.6{\tiny $\pm$8.9} &63.5{\tiny $\pm$14.6} &74.8{\tiny $\pm$11.9} &85.0 &\pmb{100.0} &86.8{\tiny $\pm$4.6} & 83.3{\tiny $\pm$5.1}  \\
&antmaze-medium-play      &71.9{\tiny $\pm$16.2} &66.3{\tiny $\pm$9.2} &70.9{\tiny $\pm$11.2} &71.4{\tiny $\pm$10.9} &78.0 & \pmb{93.3} & 84.1{\tiny $\pm$10.8} & 80.8{\tiny $\pm$4.4} \\
Locomotion tasks&antmaze-large-diverse     &20.2{\tiny $\pm$9.1} &63.9{\tiny $\pm$10.4} &50.7{\tiny $\pm$18.8} &49.0{\tiny $\pm$17.2} &82.0 & 60.0& \pmb{88.2{\tiny $\pm$5.3}}& 78.5{\tiny $\pm$4.5} \\
&antmaze-large-play   &23.1{\tiny $\pm$15.6} &64.7{\tiny $\pm$14.5} &56.5{\tiny $\pm$14.4} &63.2{\tiny $\pm$16.1} &74.0 & 66.7& \pmb{86.1{\tiny $\pm$7.5}} &79.8{\tiny $\pm$5.2} \\
&antmaze-ultra-diverse   &14.4{\tiny $\pm$9.7} &39.4{\tiny $\pm$20.6} &21.6{\tiny $\pm$15.2} &29.8{\tiny $\pm$13.6} &26.0& 33.3& 52.9{\tiny $\pm$17.4} &\pmb{60.1{\tiny $\pm$8.6}}\\
&antmaze-ultra-play     &20.7{\tiny $\pm$9.7} &38.2{\tiny $\pm$18.1} &29.8{\tiny $\pm$12.4} &31.0{\tiny $\pm$19.4} &22.0& 20.0& 39.2{\tiny $\pm$14.8} & \pmb{68.8{\tiny $\pm$4.7}} \\
\midrule
Manipulation tasks&kitchen-partial   &38.5{\tiny $\pm$11.8} &32.0{\tiny $\pm$16.7} & 39.2{\tiny $\pm$13.5} &18.4{\tiny $\pm$14.3} &- & - & 65.0{\tiny $\pm$9.2}& \pmb{69.3{\tiny $\pm$2.6}} \\
&kitchen-mixed   &46.7{\tiny $\pm$20.1} &46.8{\tiny $\pm$17.6} &51.3{\tiny $\pm$12.8} &27.9{\tiny $\pm$17.9} &-&-& \pmb{67.7{\tiny $\pm$6.8}}& 63.8{\tiny $\pm$6.7}\\
\midrule
Average  &-&37.9 &52.9 &47.9 &45.7 &61.2 &62.2 &71.3 &\pmb{73.1} \\
\bottomrule
\end{tabular}
\end{small}
\end{center}
\end{table*}

Compared with hierarchical methods, RSIQL achieves competitive overall performance while maintaining a flat policy structure. On AntMaze-Ultra, RSIQL performs particularly well, improving over HIQL on both ultra-diverse and ultra-play. This result is consistent with the motivation of RSIQL: when goal-completion information is temporally delayed, intermediate reward stimulation can provide less-delayed value-learning targets.

On medium and large AntMaze tasks, RSIQL is competitive but not uniformly better than the strongest baselines. For example, HIQL and trajectory-modeling methods can perform better on some datasets, suggesting that explicit hierarchy or trajectory-level modeling can still be advantageous when subgoal prediction is reliable. On Kitchen, RSIQL attains strong scores under the sparse reward conversion. Because the published Kitchen baselines in Table~\ref{tab:d4rl_main} use their original reward protocol, we do not interpret these cross-paper rows as a strictly controlled comparison; the matched Kitchen comparison is instead provided in Table~\ref{tab:ablation_stimulation}.

\subsection{Main Results on OGBench}
\label{sec6_ogbench}

Table~\ref{tab:ogbench_main} reports results on OGBench state-based and pixel-based tasks. RSIQL improves over GCIQL on average, indicating that reward stimulation is beneficial beyond the D4RL setting. It also achieves performance comparable to HIQL while avoiding high-level subgoal prediction at test time.

\begin{table*}[h!]
\caption{Performance of RSIQL on OGBench goal-conditioned tasks. Baseline results follow the official OGBench reference results~\citep{park2024ogbench}; RSIQL results are our six-seed runs. Bold marks the highest reported mean in each row; if all methods are tied, no entry is bolded.}
\label{tab:ogbench_main}
\begin{center}
\begin{small}
\begin{tabular}{llccccccc}
\toprule
&Dataset& GCBC & GCIVL & GCIQL  &  QRL & CRL & HIQL  &RSIQL(Ours)\\
\midrule
\multicolumn{9}{c}{\textit{State-based tasks}}\\
\midrule
&antmaze-medium-navigate   &29{\tiny $\pm$4} &72{\tiny $\pm$8} & 71{\tiny $\pm$4} &88{\tiny $\pm$3} & 95{\tiny $\pm$1}& \pmb{96{\tiny $\pm$1}}& 93.3{\tiny $\pm$2.1} \\
&antmaze-large-navigate   &24{\tiny $\pm$2} &16{\tiny $\pm$5} &34{\tiny $\pm$4} &75{\tiny $\pm$6} & 83{\tiny $\pm$4}& \pmb{91{\tiny $\pm$2}}& 88.5{\tiny $\pm$4.1}\\
&antmaze-giant-navigate   &0{\tiny $\pm$0} &0{\tiny $\pm$0} &0{\tiny $\pm$0} &0{\tiny $\pm$0} & 14{\tiny $\pm$3}& 16{\tiny $\pm$3}& \pmb{19.7{\tiny $\pm$3.2}}\\
Locomotion tasks&antmaze-medium-stitch   &45{\tiny $\pm$11} &44{\tiny $\pm$6} & 29{\tiny $\pm$6} &59{\tiny $\pm$7} & 53{\tiny $\pm$6}& \pmb{94{\tiny $\pm$1}}& 89.9{\tiny $\pm$4.4} \\
&antmaze-large-stitch   &3{\tiny $\pm$3} &18{\tiny $\pm$2} &7{\tiny $\pm$2} &18{\tiny $\pm$2} & 11{\tiny $\pm$2}& \pmb{67{\tiny $\pm$5}}& 64.4{\tiny $\pm$5.7}\\
&antmaze-giant-stitch   &0{\tiny $\pm$0} &0{\tiny $\pm$0} &0{\tiny $\pm$0} &0{\tiny $\pm$0} & 0{\tiny $\pm$0}& 2{\tiny $\pm$2}& \pmb{10.7{\tiny $\pm$5.3}}\\
\midrule
&cube-single-play-v0 & $6$ {\tiny $\pm 2$} & $53$ {\tiny $\pm 4$} & \pmb{68 {\tiny $\pm 6$}} & $5$ {\tiny $\pm 1$} & $19$ {\tiny $\pm 2$} & $15$ {\tiny $\pm 3$} & $61.3$ {\tiny $\pm 4.4$} \\
Manipulation tasks&cube-double-play-v0 & $1$ {\tiny $\pm 1$} & $36$ {\tiny $\pm 3$} & $40$ {\tiny $\pm 5$} & $1$ {\tiny $\pm 0$} & $10$ {\tiny $\pm 2$} & $6$ {\tiny $\pm 2$} & $\mathbf{46.8}$ {\tiny $\pm 5.1$} \\
&scene-play-v0 & $5$ {\tiny $\pm 1$} & $42$ {\tiny $\pm 4$} & $51$ {\tiny $\pm 4$} & $5$ {\tiny $\pm 1$} & $19$ {\tiny $\pm 2$} & $38$ {\tiny $\pm 3$} & $\mathbf{65.6}$ {\tiny $\pm 3.8$} \\
\midrule
\multicolumn{9}{c}{\textit{Pixel-based tasks}}\\
\midrule
&visual-antmaze-medium-navigate   &11{\tiny $\pm$2} &22{\tiny $\pm$2} & 11{\tiny $\pm$1} &0{\tiny $\pm$0} & \pmb{94{\tiny $\pm$1}}& 93{\tiny $\pm$4}& 89.3{\tiny $\pm$6.5} \\
&visual-antmaze-large-navigate   &4{\tiny $\pm$0} &5{\tiny $\pm$1} &4{\tiny $\pm$1} &0{\tiny $\pm$0} & \pmb{84{\tiny $\pm$1}}& 53{\tiny $\pm$9}& 50.2{\tiny $\pm$6.6}\\
&visual-antmaze-giant-navigate   &0{\tiny $\pm$0} &1{\tiny $\pm$1} &0{\tiny $\pm$0} &0{\tiny $\pm$0} & \pmb{47{\tiny $\pm$2}}& 6{\tiny $\pm$4}& 13.5{\tiny $\pm$7.3}\\
Locomotion tasks&visual-antmaze-medium-stitch   &67{\tiny $\pm$4} &6{\tiny $\pm$2} & 2{\tiny $\pm$0} &0{\tiny $\pm$0} & 69{\tiny $\pm$2}& \pmb{87{\tiny $\pm$2}}& 81.4{\tiny $\pm$5.3} \\
&visual-antmaze-large-stitch   &24{\tiny $\pm$3} &1{\tiny $\pm$1} &0{\tiny $\pm$0} &1{\tiny $\pm$1} & 11{\tiny $\pm$3}& \pmb{28{\tiny $\pm$2}}& 22.6{\tiny $\pm$5.9}\\
&visual-antmaze-giant-stitch   &0{\tiny $\pm$0} &0{\tiny $\pm$0} &0{\tiny $\pm$0} &0{\tiny $\pm$0} & 0{\tiny $\pm$0}& 0{\tiny $\pm$0}& 0{\tiny $\pm$0}\\
\midrule
&visual-cube-single-play-v0 & $5$ {\tiny $\pm 1$} & $60$ {\tiny $\pm 5$} & $30$ {\tiny $\pm 5$} & $41$ {\tiny $\pm 15$} & $31$ {\tiny $\pm 15$} & $\mathbf{89}$ {\tiny $\pm 0$} & $68.5$ {\tiny $\pm 5.2$} \\
Manipulation tasks&visual-cube-double-play-v0 & $1$ {\tiny $\pm 1$} & $10$ {\tiny $\pm 2$} & $1$ {\tiny $\pm 1$} & $5$ {\tiny $\pm 0$} & $2$ {\tiny $\pm 1$} & $\mathbf{39}$ {\tiny $\pm 2$} & $29.3$ {\tiny $\pm 4.7$} \\
&visual-scene-play-v0 & $12$ {\tiny $\pm 2$} & $25$ {\tiny $\pm 3$} & $12$ {\tiny $\pm 2$} & $10$ {\tiny $\pm 1$} & $11$ {\tiny $\pm 2$} & $\mathbf{49}$ {\tiny $\pm 4$} & $38.8$ {\tiny $\pm 6.5$} \\
\midrule
Average &- &13.2 &22.8 &20.0 &17.1 &36.3 &48.3 &\pmb{51.9} \\
\bottomrule
\end{tabular}
\end{small}
\end{center}
\end{table*}

The gains are most pronounced on tasks where intermediate progress can be reliably identified from offline trajectories. In manipulation tasks such as cube and scene environments, RSIQL substantially improves over several flat baselines, suggesting that selective reward stimulation can provide useful training signals for compositional goal reaching.

However, RSIQL is less competitive on some visual long-horizon navigation tasks. In these tasks, valid state-goal pairs can be sparse, visual representations may be harder to learn, and the auxiliary value function may provide noisier progress estimates. This limitation is consistent with the design of RSIQL: reward stimulation is most useful when the dataset contains informative intermediate states and the auxiliary value function can rank them reliably.

\subsection{Ablation Studies}
\label{sec6_ablation}

We conduct ablations to isolate the effect of reward stimulation and the design choices in RSIQL.

\paragraph{Reward-stimulation ablations.}
We compare RSIQL with several variants to isolate the role of selective progress-aware stimulation. No stimulation corresponds to the matched GCIQL baseline. Random stimulation applies reward stimulation to randomly selected transitions, using the same average stimulation frequency as RSIQL. Unfiltered $k$-step stimulation removes the auxiliary-value progress criterion and stimulates every transition for which a valid $k$-step future state $s_{t+k}$ exists. Equivalently, it sets $\eta_t(g;k,\delta)=1$ for all transitions with $t+k\leq T$. The full RSIQL method applies stimulation only when the auxiliary value function predicts that $s_{t+k}$ makes sufficient progress toward the final goal.

Table~\ref{tab:ablation_stimulation} isolates the effect of the reward-stimulation design under matched training settings. The no-stimulation GCIQL baseline obtains an average score of 38.8. Random stimulation achieves the same average score of 38.8, showing that adding intermediate reward signals at arbitrary transitions does not by itself improve performance. Unfiltered $k$-step stimulation reaches 41.9, providing only a modest improvement despite substantially increasing the density of intermediate rewards. In contrast, RSIQL achieves an average score of 72.6 and consistently improves performance across all three tasks. These results indicate that the main benefit of RSIQL does not come from reward densification alone, but from selectively stimulating transitions whose future states are estimated to make meaningful progress toward the final goal.

\begin{table}[t]
\centering
\caption{
Ablation of reward stimulation design. Results are averaged over the same seeds as the main experiments.
All variants use the same training and evaluation protocol; for Kitchen-Partial, the same sparse terminal-style reward conversion is used for every method.
No stimulation corresponds to the matched GCIQL baseline.
}
\label{tab:ablation_stimulation}
\begin{tabular}{lcccc}
\toprule
Method & AntMaze-Large-Play & AntMaze-Ultra-Play & Kitchen-Partial & Average \\
\midrule
No stimulation (GCIQL) & 53.2{\tiny $\pm$16.1} & 31.5{\tiny $\pm$11.9} & 31.7{\tiny $\pm$15.4} & 38.8 \\
Random stimulation & 49.3{\tiny $\pm$14.6} & 31.5{\tiny $\pm$11.5} & 35.7{\tiny $\pm$13.7} & 38.8 \\
Unfiltered $k$-step stimulation & 46.2{\tiny $\pm$9.8} & 38.3{\tiny $\pm$12.6} & 41.2{\tiny $\pm$8.6} & 41.9 \\
RSIQL & 79.8{\tiny $\pm$5.2} & 68.8{\tiny $\pm$4.7} & 69.3{\tiny $\pm$2.6} & 72.6 \\
\bottomrule
\end{tabular}
\end{table}

\paragraph{Effect of the subgoal interval $k$.}
The interval $k$ controls the temporal distance between the current state $s_t$ and the candidate intermediate state $s_{t+k}$. If $k$ is too small, candidate states may be too close to provide substantially less-delayed supervision. If $k$ is too large, the candidate state itself becomes temporally distant, reducing the benefit of introducing less-delayed supervision.

\begin{figure}
\begin{center}
\includegraphics[width=0.99\columnwidth]{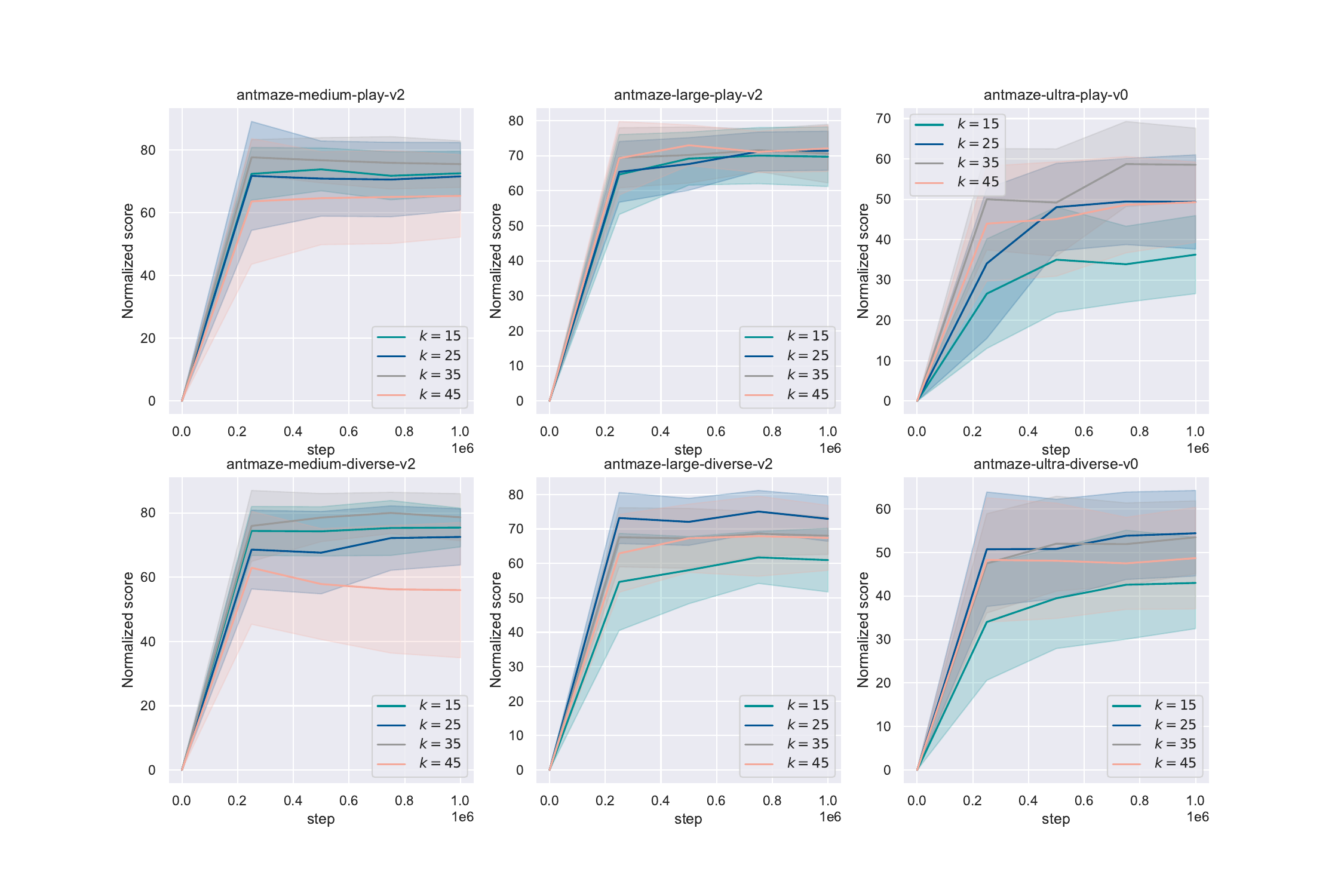}
\caption{
Training curve of AntMaze tasks using different subgoal interval $k$ across 5 random seeds with 95\% confidence intervals (shaded regions). The evaluation interval is 250000 with evaluation episode length 100.
}
\label{fig:k_antmaze}
\end{center}
\end{figure}

\begin{figure}
\begin{center}
\includegraphics[width=0.99\columnwidth]{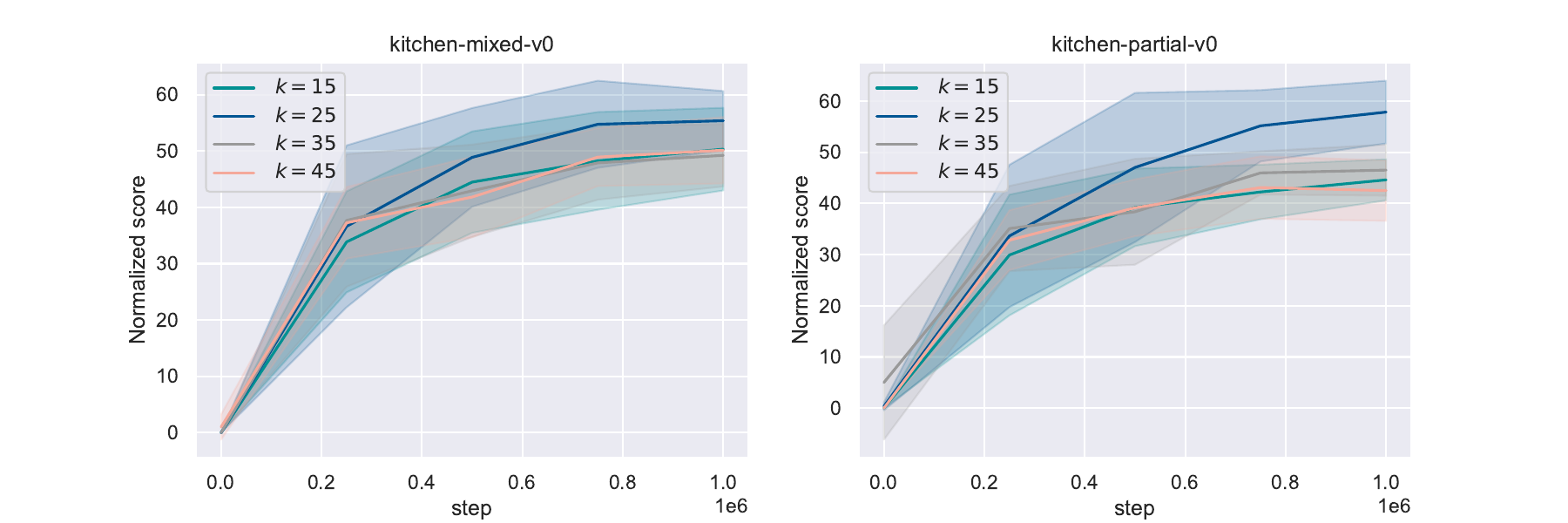}
\caption{
Training curve of Kitchen tasks using different subgoal interval $k$ across 5 random seeds with 95\% confidence intervals (shaded regions). The evaluation interval is 250000 with evaluation episode length 100.
}
\label{fig:k_kitchen}
\end{center}
\end{figure}

Figures~\ref{fig:k_antmaze} and~\ref{fig:k_kitchen} report a post-hoc sensitivity analysis over $k\in\{15,25,35,45\}$. Importantly, these experiments are not used for per-task hyperparameter selection: all main results use the single fixed value $k=25$. Performance is generally strongest for intermediate values of $k$, while both very small and very large intervals tend to be less effective. Some dataset-specific variation remains; for example, $k=35$ can perform better on several play datasets, whereas $k=25$ is often competitive on diverse datasets. These differences suggest that trajectory structure affects the most favorable temporal spacing, but the fixed choice $k=25$ provides a common setting across all reported main experiments.

In general, intermediate values perform best, while very small or very large values can reduce performance. This supports the interpretation that reward stimulation should be neither too dense nor too sparse.

\paragraph{Effect of the progress threshold $\delta$.}
The threshold $\delta\in[0,1]$ controls the strictness of progress-based subgoal selection. Larger values require a larger auxiliary-value improvement before reward stimulation is applied, whereas smaller values make the criterion more permissive. A threshold that is too small may admit weak or ineffective progress signals, while a threshold that is too large may reject useful intermediate states.
\begin{figure}
\begin{center}
\includegraphics[width=0.99\columnwidth]{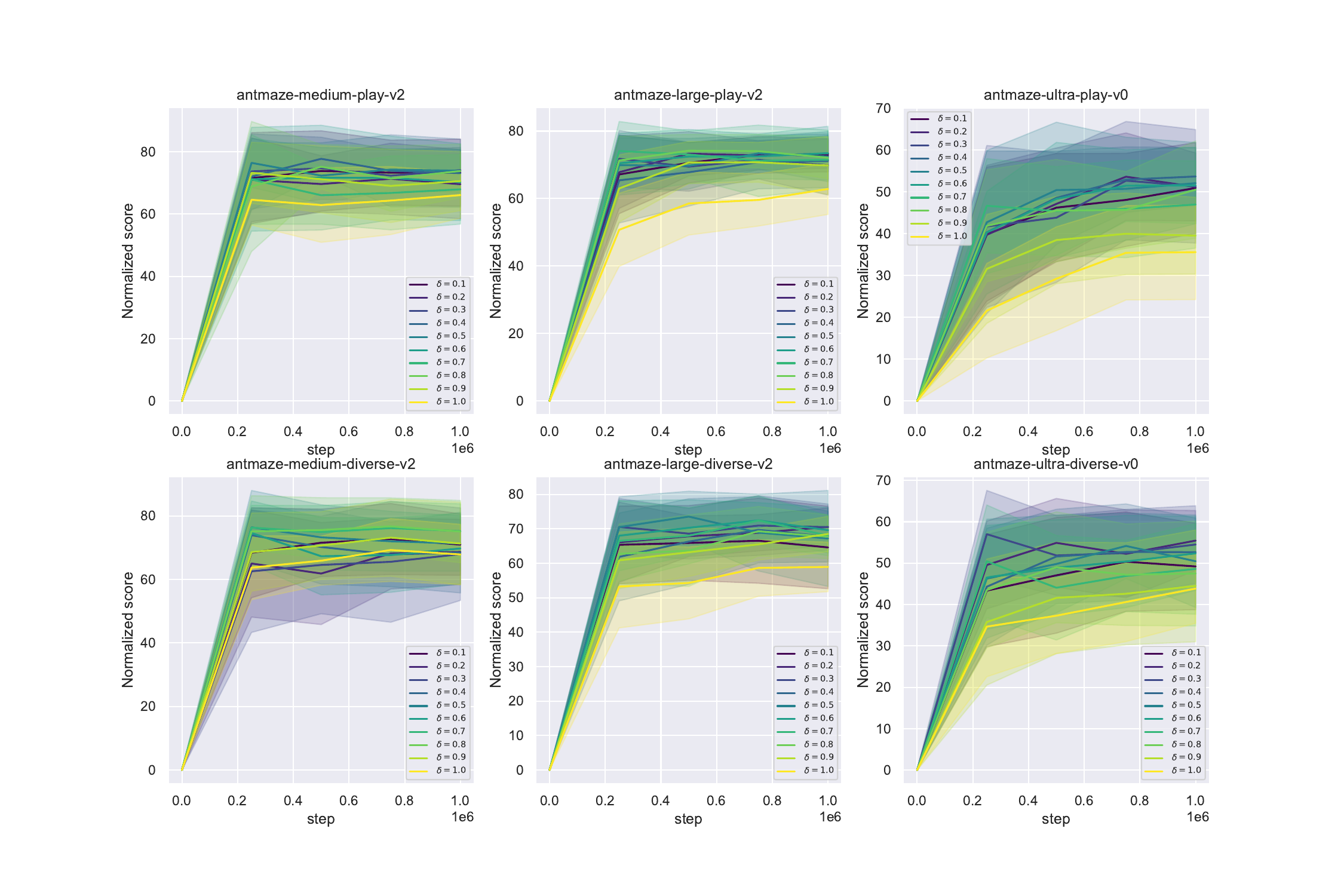}
\caption{
Training curve of AntMaze tasks using different subgoal effectiveness ratio $\delta$ across 5 random seeds with 95\% confidence intervals (shaded regions). The evaluation interval is 250000 with evaluation episode length 100.
}
\label{fig:delta_antmaze}
\end{center}
\end{figure}
\begin{figure}
\begin{center}
\includegraphics[width=0.99\columnwidth]{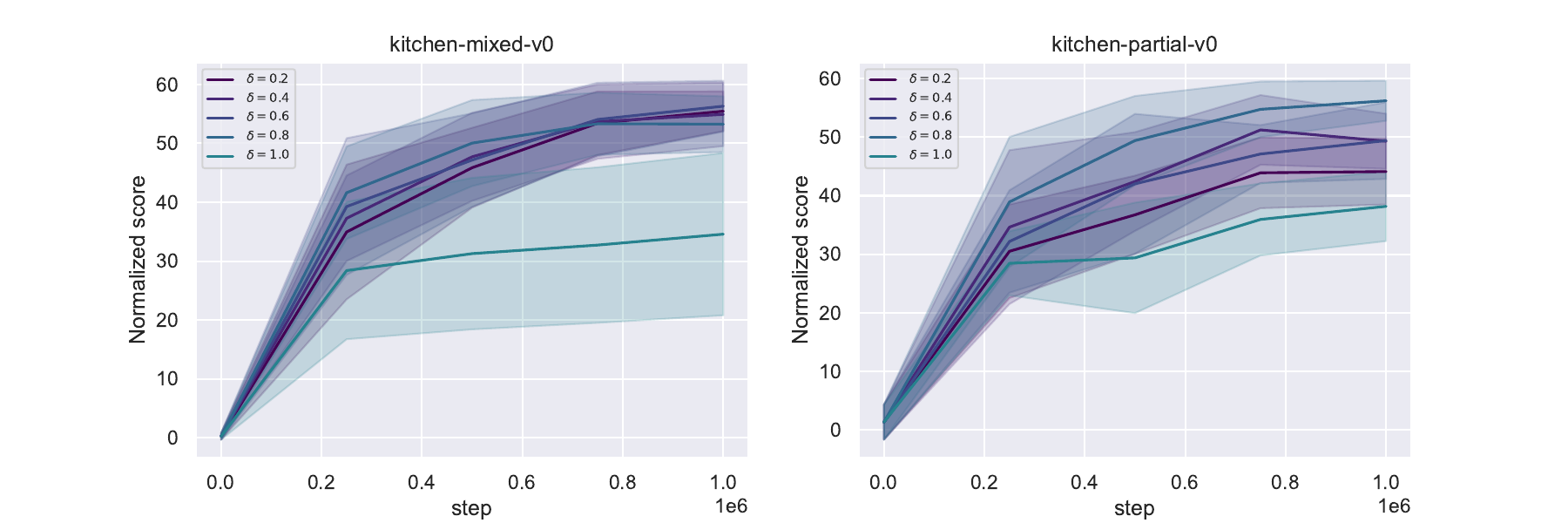}
\caption{
Training curve of Kitchen tasks using different subgoal effectiveness ratio $\delta$ across 5 random seeds with 95\% confidence intervals (shaded regions). The evaluation interval is 250000 with evaluation episode length 100.
}
\label{fig:delta_kitchen}
\end{center}
\end{figure}

Figures~\ref{fig:delta_antmaze} and~\ref{fig:delta_kitchen} present a post-hoc sensitivity analysis of this threshold. We evaluate $\delta\in\{0.1,0.2,\ldots,1.0\}$ on AntMaze and $\delta\in\{0.2,0.4,0.6,0.8,1.0\}$ on Kitchen. As with $k$, these sweeps are not used to tune the main results: all main experiments use the fixed value $\delta=0.6$. Performance tends to degrade near the extremes. When $\delta$ is close to zero, the selector becomes permissive and can stimulate transitions that provide little meaningful progress; when $\delta$ approaches one, the criterion can become overly restrictive and discard useful intermediate supervision. Across the tested tasks, performance is comparatively stable for intermediate thresholds, particularly around $\delta\in[0.4,0.8]$, supporting the use of $\delta=0.6$ as a single common setting.

\subsection{Computational Efficiency and Policy Structure}
\label{sec6_efficiency}

RSIQL avoids learning a high-level subgoal policy and does not perform subgoal prediction during evaluation. To quantify this difference, we compare the number of learned policy components and test-time inference cost with hierarchical baselines.

\begin{table}[t]
\centering
\caption{
Comparison of policy structure and computational cost.
Times are measured on the same hardware. RSIQL auxiliary pretraining is a one-time task-level cost and can be reused across repeated runs on the same task.
}
\label{tab:complexity}
\begin{tabular}{lcccccc}
\toprule
Method & High-level policy & Test-time subgoal prediction & Policy structure & Main training & Aux. pretraining & First-run total \\
\midrule
GCIQL & No & No & Flat & 2.0h & -- & 2.0h \\
HIQL & Yes & Yes & Hierarchical & 4.5h & -- & 4.5h \\
RSIQL & No & No & Flat & 2.5h & 0.33h & 2.83h \\
\bottomrule
\end{tabular}
\end{table}

Table~\ref{tab:complexity} compares policy structure and computational cost for the \textit{AntMaze-Ultra-Play} task.
RSIQL has the same flat deployment structure as GCIQL and does not require test-time subgoal prediction.
Compared with HIQL, RSIQL avoids training and deploying a high-level policy.
RSIQL introduces an auxiliary value function for progress estimation, but this network is used only during training.
In our implementation, auxiliary value pretraining takes about 20 minutes on the same task and can be reused for repeated RSIQL runs on that task.
Even including this one-time pretraining cost, RSIQL remains substantially cheaper than HIQL in our setting.

\subsection{Discussion of Failure Cases}
\label{sec6_failure}

RSIQL is most effective when offline trajectories contain intermediate states that make meaningful progress toward the final goal and when the auxiliary value function can identify such progress. Its performance can degrade in three cases.

First, if the dataset has poor goal coverage, then useful progress-making state-goal pairs may be rare. In this case, reward stimulation has few reliable transitions to select. Second, if the auxiliary value function is inaccurate, RSIQL may stimulate ineffective intermediate states, which can bias value learning toward detours. Third, in high-dimensional visual tasks, representation learning can become a bottleneck, making progress estimation less reliable.

These failure modes are consistent with the empirical results on the most difficult OGBench visual navigation tasks. They also suggest future directions, such as improving goal sampling, learning more robust progress estimators, and combining reward stimulation with stronger visual representation learning.

\section{Limitations}
\label{sec7}

RSIQL has several limitations. First, its effectiveness depends on the quality of the offline dataset. Reward stimulation is useful only when the dataset contains intermediate states that make meaningful progress toward the final goal. If goal-reaching trajectories are absent or if the dataset has poor coverage of relevant state-goal pairs, RSIQL may have few reliable transitions to stimulate.

Second, RSIQL depends on the auxiliary goal-conditioned value function used for progress estimation. Although this value function only needs to provide a coarse relative ranking between $s_t$ and $s_{t+k}$, inaccurate estimates can lead to incorrect reward stimulation. In particular, if the auxiliary value function assigns high value to detour states or fails to distinguish useful progress from irrelevant motion, RSIQL may introduce misleading training signals. This limitation is consistent with the ablation results in Section~\ref{sec6_ablation}, where random or unfiltered stimulation performs substantially worse than progress-aware stimulation.

Third, RSIQL modifies the reward used during training and therefore does not provide a policy-invariance guarantee under the original task reward. Unlike potential-based reward shaping, RSIQL is intended as a training-time credit-assignment mechanism. All policies are evaluated under the original task reward and success metric, but the stimulated training objective may in principle bias learning if progress selection is inaccurate.

Fourth, our theoretical analysis is based on a stylized finite-horizon delayed-goal setting and a local value-error model. These assumptions allow us to isolate the reward-propagation mechanism, but they do not fully characterize early goal arrival, partial observability, or representation-learning errors in high-dimensional visual tasks. Extending the analysis beyond the delayed-goal setting and local error model remains an important direction.

Finally, RSIQL is less effective when progress estimation is difficult, such as in extremely sparse multi-goal datasets or visually complex environments. In these settings, valid state-goal pairs may be rare and the auxiliary value function may provide noisy progress rankings. Future work could combine reward stimulation with improved goal sampling,
uncertainty-aware progress estimation, and stronger representation learning
for visually complex or poorly covered offline datasets.

\section{Conclusion}
\label{sec8}

We studied offline goal-conditioned reinforcement learning from a reward-propagation perspective. Our analysis identifies horizon-dependent scaling of the absolute goal-reaching message and shows, in a stylized delayed-goal setting, that success--failure value separation can become small relative to local value-estimation error at early states. These mechanism-level results motivate adding less-delayed supervision at informative intermediate states.

Based on this idea, we proposed Reward Stimulation Implicit Q-Learning (RSIQL), a simple flat offline GCRL method. RSIQL uses an auxiliary goal-conditioned value function to identify intermediate states that are estimated to make progress toward the final goal, and then applies selective reward stimulation to strengthen value learning. Unlike hierarchical methods, RSIQL does not train a high-level subgoal policy and does not require subgoal prediction at test time.

Experiments on D4RL and OGBench show that RSIQL improves over goal-conditioned IQL on average and achieves competitive performance relative to hierarchical offline GCRL methods while retaining a flat deployment structure. Ablation studies further show that the improvement comes from progress-aware reward stimulation rather than from reward densification alone: random stimulation and unfiltered $k$-step stimulation perform much worse than full RSIQL.

Overall, the results suggest that intermediate trajectory structure can be useful for long-horizon offline GCRL even without explicit hierarchical control. A promising direction for future work is to combine reward stimulation with better progress estimators, uncertainty-aware subgoal selection, and representation learning methods for visually complex or poorly covered offline datasets.

\bibliography{references}
\bibliographystyle{tmlr}

\appendix
\section{Appendix}

\subsection{Horizon-dependent scaling of the goal-reaching backward message}
\label{app_a}

In this section, we prove Proposition~\ref{prop1_revised}. We make the dependence on the goal $g$ explicit throughout the proof. Recall that the $0/-1$ goal-reaching reward convention used in the analysis is
\[
r(s_t,a_t,g)=
\begin{cases}
0, & \text{if the goal is reached},\\
-1, & \text{otherwise}.
\end{cases}
\]
Following the probabilistic-inference formulation, the optimality likelihood is
\[
p(\cO_t=1\mid s_t,a_t,g)=\exp(r(s_t,a_t,g)).
\]
Therefore, before the goal is reached, each transition contributes a multiplicative factor $e^{-1}$.

We analyze the goal-reaching component of the backward message:
\begin{align}
\cM^g(s_t,a_t,g)
=
p\!\left(s_T=g,\prod_{i=t}^{T}\cO_i=1\mid s_t,a_t,g\right),
\label{eq:app_goal_message_sa}
\end{align}
and its state-level counterpart
\begin{align}
\cM^g(s_t,g)
=
\int_{\cA}
\cM^g(s_t,a_t,g)\pi_\beta(a_t\mid s_t,g)\,da_t.
\label{eq:app_goal_message_s}
\end{align}
This notation isolates the contribution of trajectories that terminate by reaching the target goal at time $T$. For this analysis, successful trajectories are stopped when the goal is first reached. Hence, conditional on success at terminal time $T$, the transitions from $t$ through $T-1$ are non-goal transitions, while the terminal state satisfies $s_T=g$. The derivation does not require deterministic transition dynamics; stochasticity is absorbed into the goal-reaching probability under the behavior policy and environment dynamics.

We first derive a recursive expression. By the Markov property,
\begin{align}
\cM^g(s_t,a_t,g)
&=
p(\cO_t=1\mid s_t,a_t,g)
\int_{\cS}
\cM^g(s_{t+1},g)
p(s_{t+1}\mid s_t,a_t,g)\,ds_{t+1}.
\label{eq:app_rec_sa}
\end{align}
Similarly,
\begin{align}
\cM^g(s_t,g)
=
\int_{\cA}
\cM^g(s_t,a_t,g)
\pi_\beta(a_t\mid s_t,g)\,da_t.
\label{eq:app_rec_s}
\end{align}

\begin{proof}
At the terminal step, if $s_T=g$, then $r(s_T,a_T,g)=0$ and hence
\begin{align}
p(\cO_T=1\mid s_T,a_T,g)=\exp(0)=1.
\end{align}
Therefore,
\begin{align}
\cM^g(s_T,a_T,g)=1,
\qquad
\cM^g(s_T,g)=1.
\end{align}

We now prove the result by backward induction.

\paragraph{Base case.}
For $t=T-1$, since the transition at $T-1$ is non-terminal before reaching the goal, we have
\[
p(\cO_{T-1}=1\mid s_{T-1},a_{T-1},g)=e^{-1}.
\]
Using \eqnref{eq:app_rec_sa},
\begin{align}
\cM^g(s_{T-1},a_{T-1},g)
&=
e^{-1}
\int_{\cS}
\cM^g(s_T,g)
p(s_T\mid s_{T-1},a_{T-1},g)\,ds_T \nonumber\\
&=
e^{-1}
\PP_{\pi_\beta}(s_T=g\mid s_{T-1},a_{T-1},g).
\end{align}
Similarly,
\begin{align}
\cM^g(s_{T-1},g)
&=
\int_{\cA}
\cM^g(s_{T-1},a_{T-1},g)
\pi_\beta(a_{T-1}\mid s_{T-1},g)\,da_{T-1} \nonumber\\
&=
e^{-1}
\PP_{\pi_\beta}(s_T=g\mid s_{T-1},g).
\end{align}

\paragraph{Induction hypothesis.}
Assume that for some $t+1<T$,
\begin{align}
\cM^g(s_{t+1},a_{t+1},g)
=
\PP_{\pi_\beta}(s_T=g\mid s_{t+1},a_{t+1},g)e^{-(T-t-1)},
\end{align}
and
\begin{align}
\cM^g(s_{t+1},g)
=
\PP_{\pi_\beta}(s_T=g\mid s_{t+1},g)e^{-(T-t-1)}.
\end{align}

\paragraph{Induction step.}
Using the recursion in \eqnref{eq:app_rec_sa},
\begin{align}
\cM^g(s_t,a_t,g)
&=
e^{-1}
\int_{\cS}
\cM^g(s_{t+1},g)
p(s_{t+1}\mid s_t,a_t,g)\,ds_{t+1} \nonumber\\
&=
e^{-1}
\int_{\cS}
\PP_{\pi_\beta}(s_T=g\mid s_{t+1},g)
e^{-(T-t-1)}
p(s_{t+1}\mid s_t,a_t,g)\,ds_{t+1} \nonumber\\
&=
e^{-(T-t)}
\PP_{\pi_\beta}(s_T=g\mid s_t,a_t,g).
\end{align}
The final equality follows from the law of total probability.

For the state-level message,
\begin{align}
\cM^g(s_t,g)
&=
\int_{\cA}
\cM^g(s_t,a_t,g)\pi_\beta(a_t\mid s_t,g)\,da_t \nonumber\\
&=
e^{-(T-t)}
\int_{\cA}
\PP_{\pi_\beta}(s_T=g\mid s_t,a_t,g)
\pi_\beta(a_t\mid s_t,g)\,da_t \nonumber\\
&=
e^{-(T-t)}
\PP_{\pi_\beta}(s_T=g\mid s_t,g).
\end{align}
Thus, by backward induction,
\begin{align}
\cM^g(s_t,a_t,g)
=
\PP_{\pi_\beta}(s_T=g\mid s_t,a_t,g)e^{-(T-t)},
\end{align}
and
\begin{align}
\cM^g(s_t,g)
=
\PP_{\pi_\beta}(s_T=g\mid s_t,g)e^{-(T-t)}.
\end{align}
This proves Proposition~\ref{prop1_revised}.
\end{proof}

\begin{remark}
The proof concerns the \emph{absolute} goal-reaching component of the backward message. The factor $e^{-(T-t)}$ is shared by actions evaluated at the same time step, so this factorization alone does not imply that relative action preferences, posterior action ratios, or learned IQL advantages decay with the horizon. Those quantities also depend on the action-dependent reachability term and on value-estimation error.
\end{remark}

\begin{remark}
The exact factor $e^{-(T-t)}$ is induced by the $0/-1$ reward scale in the control-as-inference transformation. A different non-goal reward scale changes this factor. Proposition~\ref{prop1_revised} should therefore be read as a representation-level statement about horizon-dependent scaling of the absolute message, rather than as a general failure theorem for long-horizon offline GCRL.
\end{remark}

\subsection{Optimality-conditioned policy extraction}
\label{app_optpolicy}

In this section, we prove Proposition~\ref{prop2_revised}. The proof follows the probabilistic-inference view of reinforcement learning, but we make the goal conditioning explicit.

\begin{proof}
For a state $s_t$ and goal $g$ with $\cM(s_t,g)>0$, consider the posterior action distribution conditioned on future optimality:
\begin{align}
p\!\left(a_t\mid s_t,g,\prod_{i=t}^{T}\cO_i=1\right).
\end{align}
Using Bayes' rule,
\begin{align}
p\!\left(a_t\mid s_t,g,\prod_{i=t}^{T}\cO_i=1\right)
&=
\frac{
p\!\left(\prod_{i=t}^{T}\cO_i=1\mid s_t,a_t,g\right)
p(a_t\mid s_t,g)
}{
p\!\left(\prod_{i=t}^{T}\cO_i=1\mid s_t,g\right)
}.
\end{align}
In the offline setting, the action prior is the behavior policy, so
\[
p(a_t\mid s_t,g)=\pi_\beta(a_t\mid s_t,g).
\]
By the definitions of the backward messages,
\begin{align}
p\!\left(a_t\mid s_t,g,\prod_{i=t}^{T}\cO_i=1\right)
=
\frac{\cM(s_t,a_t,g)}{\cM(s_t,g)}
\pi_\beta(a_t\mid s_t,g).
\end{align}

Using the logarithmic message transformation
\[
Q^{\mathrm{msg}}(s_t,a_t,g)=\log \cM(s_t,a_t,g),
\qquad
V^{\mathrm{msg}}(s_t,g)=\log \cM(s_t,g),
\]
we obtain
\begin{align}
p\!\left(a_t\mid s_t,g,\prod_{i=t}^{T}\cO_i=1\right)
&=
\exp\!\left(Q^{\mathrm{msg}}(s_t,a_t,g)-V^{\mathrm{msg}}(s_t,g)\right)
\pi_\beta(a_t\mid s_t,g).
\label{eq:app_posterior_policy}
\end{align}
Because
\[
\cM(s_t,g)
=
\int_{\cA}\cM(s_t,a_t,g)\pi_\beta(a_t\mid s_t,g)\,da_t,
\]
the posterior distribution in \eqnref{eq:app_posterior_policy} is normalized over actions in the support of the behavior policy.

Let $\pi(a\mid s,g)$ be a parametric policy. To fit $\pi$ to the optimality-conditioned posterior, we minimize
\begin{align}
\EE_{s,g}
\left[
D_{\mathrm{KL}}
\left(
p(\cdot\mid s,g,\prod_{i=t}^{T}\cO_i=1)
\,\|\, 
\pi(\cdot\mid s,g)
\right)
\right].
\end{align}
Dropping terms independent of $\pi$, this is equivalent to maximizing
\begin{align}
\EE_{(s,a,g)\sim\cD}
\left[
\exp\!\left(Q^{\mathrm{msg}}(s,a,g)-V^{\mathrm{msg}}(s,g)\right)
\log \pi(a\mid s,g)
\right].
\end{align}
This gives the objective in Proposition~\ref{prop2_revised}.
\end{proof}

\begin{remark}
The policy in Proposition~\ref{prop2_revised} should be interpreted as an optimality-conditioned posterior policy over the support of the offline data, rather than as an unrestricted online optimal policy. This distinction is important in offline RL because actions outside the dataset support may lead to extrapolation error.
\end{remark}

\subsection{Delayed-goal value separation and local signal quality}
\label{app_b}

We first prove Lemma~\ref{lem:delayed_goal_separation}, which derives the temporal scaling of value separation directly from the delayed-goal reward structure. We then prove Proposition~\ref{prop3_revised} under the local Gumbel error model.

\begin{proof}[Proof of Lemma~\ref{lem:delayed_goal_separation}]
Under the $0/-1$ goal-reaching reward convention and the delayed-goal condition, both continuation policies receive reward $-1$ at every time $j=t,\ldots,T-1$. Therefore, the pre-terminal parts of their expected returns are identical and cancel in the difference. At time $T$, the expected reward under a continuation policy $\pi$ is
\begin{align}
\EE_\pi[r_T\mid s_t,a_t,g]
&=
0\cdot \PP_\pi(s_T=g\mid s_t,a_t,g)
-1\cdot \PP_\pi(s_T\neq g\mid s_t,a_t,g)\\
&=
-1+\PP_\pi(s_T=g\mid s_t,a_t,g).
\end{align}
Hence,
\begin{align}
Q_t^{\pi^s}(s_t,a_t,g)-Q_t^{\pi^f}(s_t,a_t,g)
&=
\gamma^{T-t}
\left[
(-1+p_s)-(-1+p_f)
\right]\\
&=
\gamma^{T-t}(p_s-p_f).
\end{align}
Taking absolute values gives
\begin{align}
\cR_t(s_t,a_t,g)
=
\gamma^{T-t}|p_s-p_f|
\le
\gamma^{T-t},
\end{align}
since $p_s,p_f\in[0,1]$. If $p_s=1$ and $p_f=0$, equality holds. This proves Lemma~\ref{lem:delayed_goal_separation}.
\end{proof}

\begin{remark}
The delayed-goal condition is essential for the geometric upper bound. If one continuation policy can reach the goal before time $T$, its reward stream differs earlier and the corresponding return gap can exceed $\gamma^{T-t}$. Thus, Lemma~\ref{lem:delayed_goal_separation} should be interpreted as a delayed-credit-assignment result rather than a universal bound for arbitrary pairs of policies.
\end{remark}

\begin{proof}[Proof of Proposition~\ref{prop3_revised}]
Let $Q_t$ denote the target action-value quantity being approximated locally and let $\hat Q_t$ be the learned Q-function. We decompose the local value-estimation error into approximation error and Bellman error:
\begin{align}
\epsilon_t(s,a,g)
&=
Q_t(s,a,g)-\hat Q_t(s,a,g),\\
\delta_t(s,a,g)
&=
\hat Q_t(s,a,g)-\cT_g\hat Q_t(s,a,g).
\end{align}
For the local probability calculation, we use the dominant local residual model
\begin{align}
\zeta_t(s,a,g)
=
\left|
\epsilon_t(s,a,g)+|\delta_t(s,a,g)|
\right|.
\label{eq:app_local_residual}
\end{align}
For a stylized local calculation, we adopt a Gumbel approximation motivated by extreme-value models used in offline RL~\citep{garg2023extreme} and assume
\[
\epsilon_t(s,a,g)\sim \cG(0,b),
\]
where $\cG(0,b)$ is a Gumbel distribution with location $0$ and scale $b>0$. Treating $\delta_t(s,a,g)$ as fixed in the local analysis, define
\[
\cZ_t
=
\epsilon_t(s,a,g)+|\delta_t(s,a,g)|.
\]
Then
\[
\cZ_t\sim \cG(|\delta_t(s,a,g)|,b).
\]

The signal-to-noise ratio is
\[
\Delta_t(s,a,g)
=
\frac{\cR_t(s,a,g)}{\zeta_t(s,a,g)}.
\]
Under the local residual model in \eqnref{eq:app_local_residual},
\[
\Delta_t(s,a,g)>1
\quad\Longleftrightarrow\quad
|\cZ_t|<\cR_t(s,a,g).
\]
Therefore,
\begin{align}
\PP(\Delta_t(s,a,g)>1)
&=
\PP(|\cZ_t|<\cR_t(s,a,g)) \nonumber\\
&=
F(\cR_t(s,a,g);|\delta_t(s,a,g)|,b)
-
F(-\cR_t(s,a,g);|\delta_t(s,a,g)|,b),
\end{align}
where the Gumbel cumulative distribution function is
\[
F(x;\mu,b)=\exp\left(-\exp\left(-\frac{x-\mu}{b}\right)\right).
\]

Let
\[
f(u)=\exp(-\exp(-u)).
\]
Its derivative is
\[
f'(u)=\exp(-u)\exp(-\exp(-u)),
\]
whose maximum is $1/e$. Thus, $f$ is globally Lipschitz with constant $1/e$. Hence,
\begin{align}
\PP(\Delta_t(s,a,g)>1)
&=
\left|
f\!\left(\frac{\cR_t(s,a,g)-|\delta_t(s,a,g)|}{b}\right)
-
f\!\left(\frac{-\cR_t(s,a,g)-|\delta_t(s,a,g)|}{b}\right)
\right| \nonumber\\
&\le
\frac{1}{e}
\left|
\frac{\cR_t(s,a,g)-|\delta_t(s,a,g)|}{b}
-
\frac{-\cR_t(s,a,g)-|\delta_t(s,a,g)|}{b}
\right| \nonumber\\
&=
\frac{2}{e b}\cR_t(s,a,g).
\end{align}
By Lemma~\ref{lem:delayed_goal_separation},
\begin{align}
\cR_t(s,a,g)\le \gamma^{T-t},
\end{align}
and therefore
\begin{align}
\PP(\Delta_t(s,a,g)>1)
\le
\frac{2}{e b}\gamma^{T-t}.
\end{align}
If $b\ge 2$, then
\begin{align}
\PP(\Delta_t(s,a,g)>1)
\le
\frac{\gamma^{T-t}}{e}.
\end{align}
This proves Proposition~\ref{prop3_revised}.
\end{proof}

\begin{remark}
Proposition~\ref{prop3_revised} is a local signal-quality result under the stated delayed-goal and Gumbel error conditions. Its horizon dependence is inherited from Lemma~\ref{lem:delayed_goal_separation}, not from Proposition~\ref{prop1_revised}. The latter concerns an absolute control-as-inference message and remains conceptually distinct from the return-separation argument used here.
\end{remark}

\begin{remark}
The role of Proposition~\ref{prop3_revised} is motivational: in delayed-goal problems, useful goal-directed value separation can become small at early states, allowing local estimation error to obscure it. RSIQL addresses this failure mode by introducing less-delayed supervision at selected progress-making transitions.
\end{remark}

\subsection{Effect of reward stimulation}
\label{app_reward_stimulation}

In this section, we prove Proposition~\ref{thm:rsiql_effect}. Recall that RSIQL constructs a stimulated reward
\begin{align}
\tilde r_j
=
r_j+ \eta_j(g;k,\delta)(r_{\max}-r_j),
\label{eq:app_stimulated_reward}
\end{align}
where $\eta_j(g;k,\delta)\in\{0,1\}$ is the progress indicator, and $r_{\max}=0$ under the $0/-1$ goal-reaching step-cost convention. The indicator $\eta_j(g;k,\delta)$ is equal to one when the $k$-step future state is estimated by the auxiliary value function to make sufficient progress toward the final goal. A selected non-terminal transition with $\tilde r_j=0$ is not made terminal: RSIQL continues to bootstrap from its successor state with the standard one-step target $\tilde r_j+\gamma V_\psi(s_{j+1},g)$. Hence stimulation removes the selected current-step penalty but leaves the discounted continuation value intact.

Let
\begin{align}
Q_t^\pi(s_t,a_t,g)
=
\mathbb{E}_{\pi}
\left[
\sum_{j=t}^{T}
\gamma^{j-t} r_j
\mid s_t,a_t,g
\right]
\end{align}
denote the action-value function under the original reward, and let
\begin{align}
\tilde Q_t^\pi(s_t,a_t,g)
=
\mathbb{E}_{\pi}
\left[
\sum_{j=t}^{T}
\gamma^{j-t} \tilde r_j
\mid s_t,a_t,g
\right]
\end{align}
denote the action-value function under the stimulated reward.

\begin{proof}
Subtracting the original return from the stimulated return gives
\begin{align}
\tilde Q_t^\pi(s_t,a_t,g)-Q_t^\pi(s_t,a_t,g)
&=
\mathbb{E}_{\pi}
\left[
\sum_{j=t}^{T}
\gamma^{j-t}
(\tilde r_j-r_j)
\mid s_t,a_t,g
\right].
\end{align}
Using the definition of $\tilde r_j$ in \eqnref{eq:app_stimulated_reward}, we obtain
\begin{align}
\tilde r_j-r_j
=
 \eta_j(g;k,\delta)(r_{\max}-r_j).
\end{align}
Therefore,
\begin{align}
\tilde Q_t^\pi(s_t,a_t,g)-Q_t^\pi(s_t,a_t,g)
&=
\mathbb{E}_{\pi}
\left[
\sum_{j=t}^{T}
\gamma^{j-t}
 \eta_j(g;k,\delta)(r_{\max}-r_j)
\mid s_t,a_t,g
\right].
\label{eq:app_return_gap}
\end{align}

Since $\eta_j(g;k,\delta)\in\{0,1\}$, $\gamma^{j-t}\ge 0$, and $r_{\max}\ge r_j$, every term inside the summation is nonnegative. Hence,
\begin{align}
\tilde Q_t^\pi(s_t,a_t,g)\ge Q_t^\pi(s_t,a_t,g).
\end{align}
Moreover, the inequality is strict whenever there exists at least one time step $j\ge t$ such that
\[
\mathbb{P}_{\pi}
\left(
\eta_j(g;k,\delta)=1
\;\text{and}\;
r_j<r_{\max}
\mid s_t,a_t,g
\right)>0.
\]
This condition means that, with nonzero probability, the trajectory encounters a selected progress-making transition that receives positive reward stimulation. This proves Proposition~\ref{thm:rsiql_effect}.
\end{proof}

We then provide the proof of Corollary~\ref{cor:shorter_horizon}.

\begin{proof}
Under the delayed-goal condition in the corollary, the original reward remains at the non-goal penalty through time $T-1$ and first receives the goal-completion improvement at time $T$. By \eqnref{eq:stimulated_reward_general}, when $\eta_{t_i}(g;k,\delta)=1$ and $r_{t_i}<r_{\max}$, reward stimulation adds $r_{\max}-r_{t_i}>0$ at time $t_i$. Its discounted contribution to the return from time $t$ is
\[
\gamma^{t_i-t}(r_{\max}-r_{t_i}).
\]
Under the $0/-1$ goal-reaching convention, this selected non-goal transition has $r_{t_i}=-1$ and $r_{\max}=0$, so the added contribution has magnitude $\gamma^{t_i-t}$. The original goal-completion improvement over the non-goal step penalty appears at time $T$ with discounted magnitude $\gamma^{T-t}$. Because $t_i<T$ and $\gamma\in(0,1)$,
\[
\gamma^{t_i-t}>\gamma^{T-t}.
\]
Hence the stimulated reward supplies a positive training contribution at a smaller temporal delay. The Bellman update itself remains one-step; only the location of the reward information is changed.
\end{proof}

\begin{remark}
Proposition~\ref{thm:rsiql_effect} is a statement about the training objective induced by the stimulated reward. It does not imply that the optimal value under the original task reward is increased. Rather, it shows that reward stimulation gives larger training returns to trajectories containing selected progress-making transitions. This stronger training signal is then used by IQL to form value targets and advantage weights for policy extraction.
\end{remark}

\begin{remark}
RSIQL should not be interpreted as potential-based reward shaping. Potential-based shaping is designed to preserve optimal policies under specific reward transformations. RSIQL instead uses selective training-time reward stimulation to improve credit assignment in offline sparse-reward goal-conditioned tasks. The final policy is still evaluated using the original goal-reaching metric.
\end{remark}

\subsection{Experiment Details}\label{app_experiment_details}

The learning procedure of RSIQL is summarized in Algorithm~\ref{alg:RSIQL}. Our offline GCRL framework is built on the goal-conditioned IQL algorithm described in Section~\ref{seciql}. We implement RSIQL on an NVIDIA 4090 GPU. Training is computationally efficient, since the method avoids both expensive graph construction in state-based settings and explicit subgoal prediction.

We use one fixed set of core IQL and RSIQL hyperparameters in the main experiments: expectile $\tau=0.7$, inverse temperature $\beta=3$, subgoal interval $k=25$, and progress threshold $\delta=0.6$. For visual goal-conditioned tasks, we additionally use the shared IMPALA state-goal encoder described in Section~\ref{sec6_setup}. The full set of optimization, architecture, and evaluation hyperparameters is reported in Table~\ref{tab4}. The alternative $k$ and $\delta$ values in Section~\ref{sec6_ablation} are used only for post-hoc sensitivity analysis and are not used to select the reported main results.

\paragraph{Auxiliary-value pretraining.}
For each task, the auxiliary progress value function $\bar V_\omega$ is pretrained from the same offline dataset $\mathcal D$ using the GC-IVL objective before RSIQL training begins. It is then kept fixed throughout RSIQL training and is never updated using stimulated rewards. This is a one-time task-level pretraining stage, and the resulting auxiliary value model can be reused across repeated RSIQL runs on the same task. Its role is only to rank candidate future states for the progress criterion; the deployed policy does not query $\bar V_\omega$ at test time.

\begin{table*}[t]
\caption{The hyperparameters for offline RL training.}
\label{tab4}
\begin{center}
\begin{small}
\begin{tabular}{lll}
\toprule
 & Hyperparameter & Value\\
\midrule
      &Actor hidden dim & 512  \\
Architecture     &Actor layers & 3 \\
   &Critic hidden dim & 512 \\
   &Critic layers & 3 \\
\midrule  
     &Batch size & 1024 for state-based tasks,256 for pixel-based tasks\\
     &Optimizer & Adam  \\
     &Actor learning rate  & 3e-4  \\
     &Critic learning rate  & 3e-4   \\
     &value learning rate  & 3e-4   \\
     &Discount factor  & 0.99    \\
IQL basic  &Target update rate for Q target networks & 0.005\\
      &Critic activation & gelu\\
      &Evaluation episode length& 100 for D4RL, 50 for OGBench \\
      &Number of iterations  & 1000000 for state-based tasks, 500000 for pixel-based tasks\\
      &expectile $\tau$  & 0.7\\
      &inverse temperature for AWR policy loss $\beta$  & 3\\
\midrule
        &subgoal interval $k$ & 25 \\
    RSIQL specific      &subgoal effectiveness ratio $\delta$ & 0.6 \\
\bottomrule
\end{tabular}
\end{small}
\end{center}
\end{table*}

We also provide representative trajectory visualizations generated by the learned policy on D4RL AntMaze tasks for different subgoal intervals in Figure~\ref{fig:traj_appendix}.\footnote{The star denotes the end of a trajectory.}

\begin{figure*}[t]
\centering
\begin{tabular}{cc}
\includegraphics[width=0.45\textwidth]{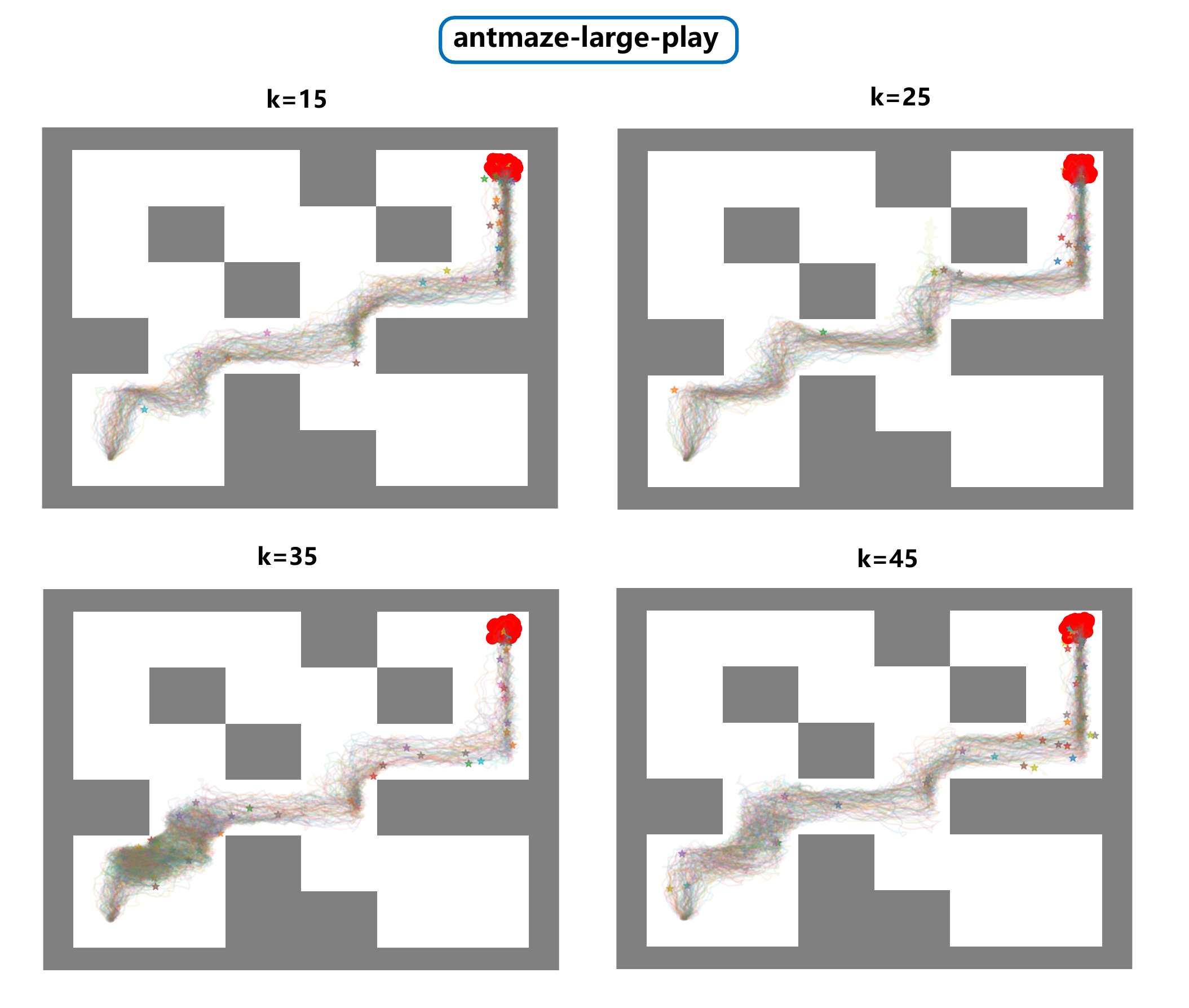} &
\includegraphics[width=0.45\textwidth]{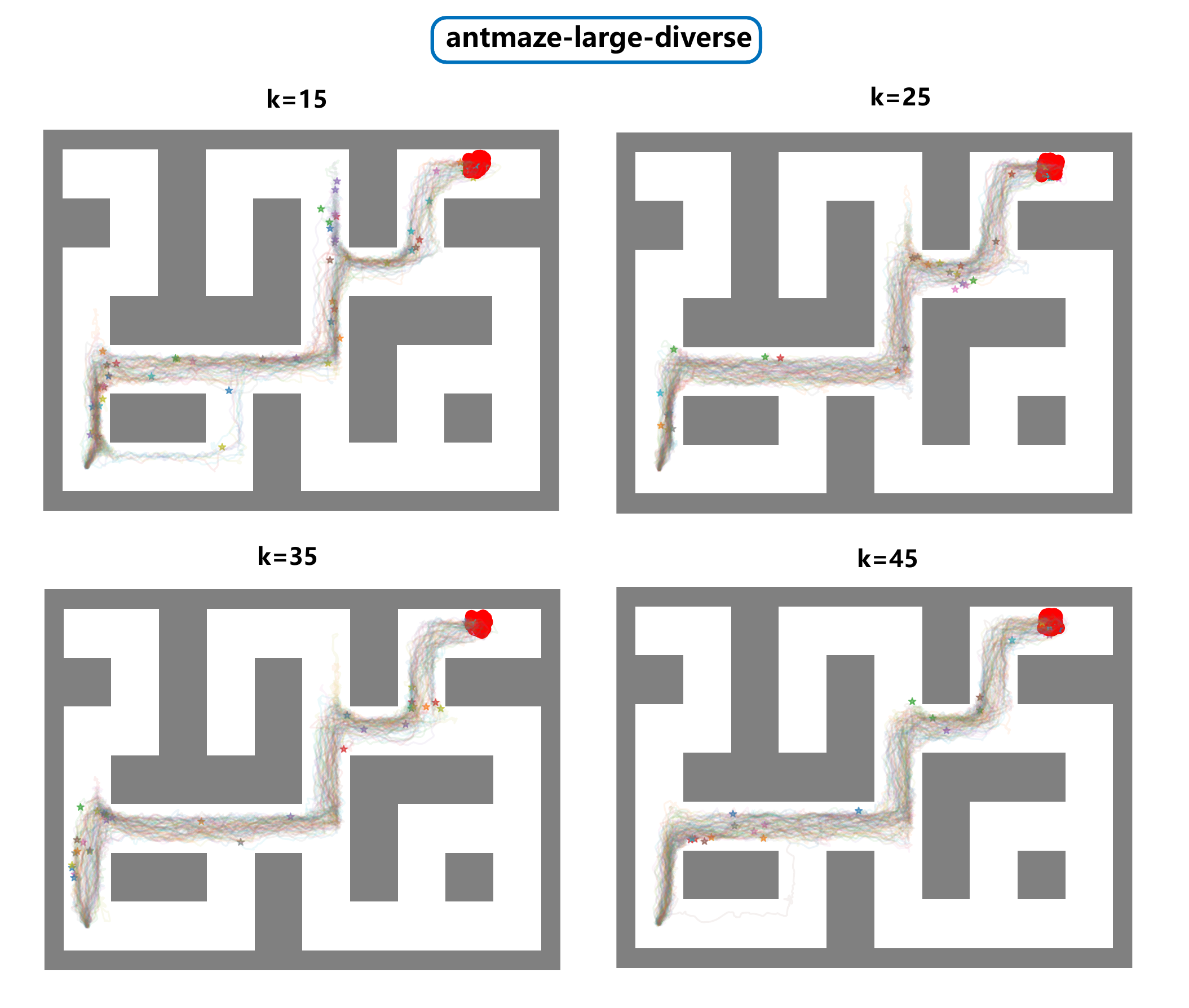} \\
\includegraphics[width=0.45\textwidth]{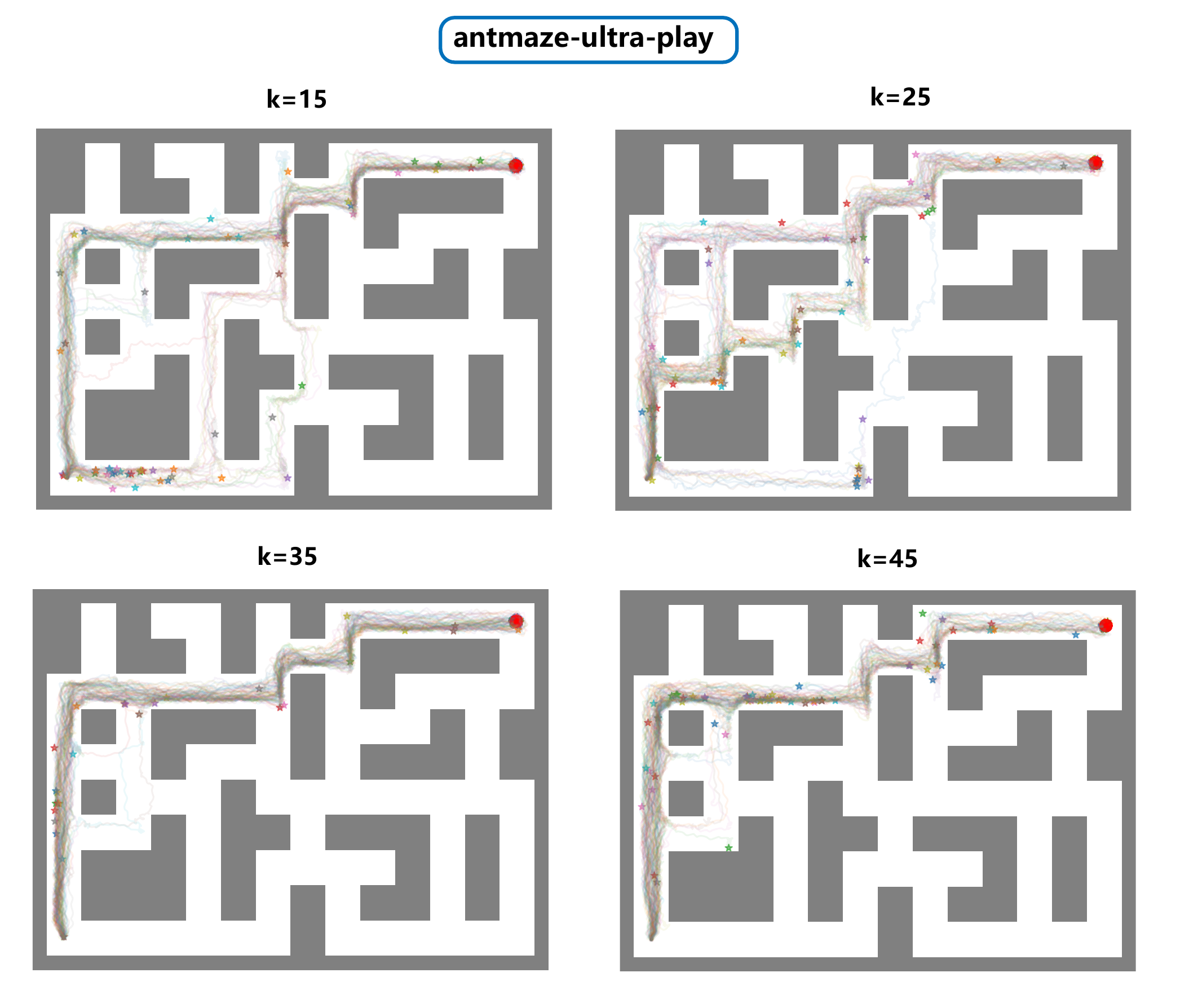} &
\includegraphics[width=0.45\textwidth]{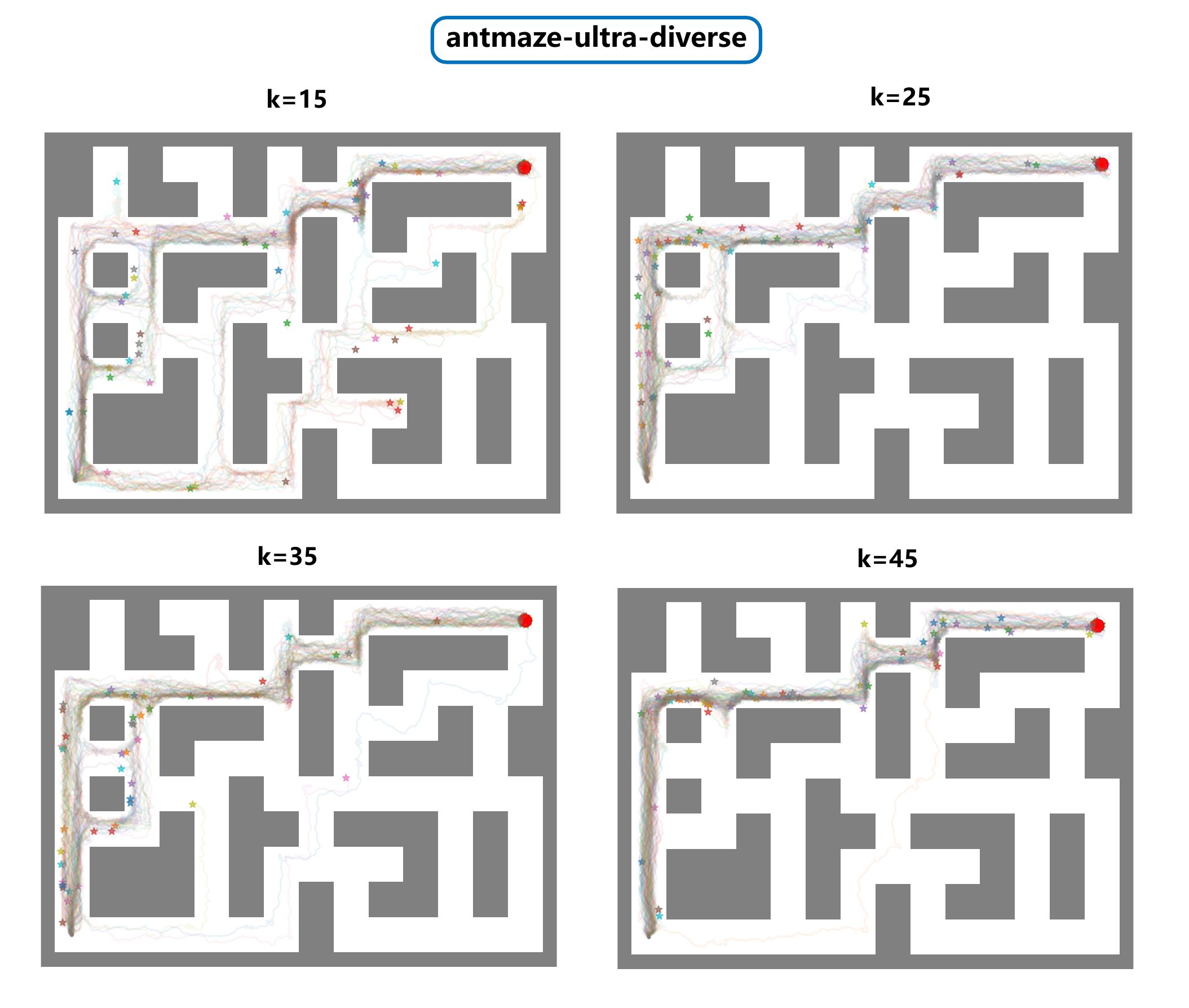}
\end{tabular}
\caption{Representative AntMaze trajectory visualizations for the subgoal-interval sensitivity study. The four panels collect the previously separate trajectory figures into a single appendix figure; task and interval labels are shown inside the original panels. The evaluation interval is 250000 with episode length 100.}
\label{fig:traj_appendix}
\end{figure*}

\end{document}